\newcommand{\CLASSINPUTtoptextmargin}{0.75in}
\newcommand{\CLASSINPUTbottomtextmargin}{0.75in}
\newcommand{\CLASSINPUTinnersidemargin}{0.75in}
\newcommand{\CLASSINPUToutersidemargin}{0.75in}
\documentclass[letterpaper, 10 pt, conference]{ieeetran}  

\IEEEoverridecommandlockouts                              

\usepackage{xspace}

\usepackage[T1]{fontenc}

\usepackage{multicol}
\usepackage{hyperref}
\hypersetup{bookmarksopen,bookmarksnumbered,
pdfpagemode=UseOutlines,
colorlinks=true,
linkcolor=blue,
anchorcolor=blue,
citecolor=blue,
filecolor=blue,
menucolor=blue,
urlcolor=blue
}
\usepackage{amsmath, amssymb, amsfonts}
\usepackage{ntheorem}
\usepackage{bm}
\usepackage{flushend}
\usepackage{graphicx}
\usepackage{mathtools}

\newcommand{\eref}[1]{Eq.~(\ref{#1})} 
\newcommand{\sref}[1]{Sec.~\ref{#1}} 
\newcommand{\figref}[1]{Fig.~\ref{#1}} 

\usepackage[font=footnotesize]{caption} 
\usepackage[font=footnotesize]{subcaption}
\usepackage{mathabx}
\usepackage{verbatim}
\usepackage{algorithm}
\usepackage[noend]{algpseudocode}

\usepackage{multirow}
\usepackage{booktabs}
\usepackage[dvipsnames]{xcolor}

\usepackage{tikz}

\newcommand{\shushman}[1]{{\xxnote{ShC}{blue}{#1}}}
\newcommand{\oren}[1]{{\xxnote{OS}{magenta}{#1}}}

\newcommand{\xxnote}[3]{}
\ifx\hidenotes\undefined
 \usepackage{color}
 \renewcommand{\xxnote}[3]{\color{#2}{#1: #3}}
\fi

\newtheorem{theorem}{Theorem}
\newtheorem*{theorem-non}{Theorem}

\newtheorem*{lemm-non}{Lemma}

\title{\vspace{0.20in}\LARGE \bf
Densification Strategies for Anytime Motion Planning 
over Large Dense Roadmaps 
}

\author{Shushman Choudhury, 
		Oren Salzman,
		Sanjiban Choudhury and 
		Siddhartha S.~Srinivasa$^{1}$
\thanks{*Work by 
Sh. C., O. S. and S. S. was (partially) funded by the National Science Foundation IIS (\#1409003), Toyota Motor Engineering~\& Manufacturing (TEMA), and the Office of Naval Research.}
\thanks{$^{1}$ The Robotics Institute, Carnegie Mellon University
		{\tt\small 
			$\{$shushmac, osalzman, sanjibac, ss5$\}$ @andrew.cmu.edu
		}}%
}

\newcommand{\calG}{\ensuremath{\mathcal{G}}\xspace}

\newcommand{\calM}{\ensuremath{\mathcal{M}}\xspace}
\newcommand{\calX}{\ensuremath{\mathcal{X}}\xspace}
\newcommand{\calQ}{\ensuremath{\mathcal{Q}}\xspace}
\newcommand{\calS}{\ensuremath{\mathcal{S}}\xspace}

\newcommand{\R}{\mathbb{R}}

\newcommand{\Cfree}{\ensuremath{\calX_{\rm free}}\xspace}

\newcommand{\Cobs}{\ensuremath{\calX_{\rm obs}}\xspace}
\newcommand{\Cs}{C-space\xspace}

\newcommand{\ignore}[1]{}

\newcommand{\cbest}[1]{c_{\mathrm{best}}^{#1}}
\newcommand{\cmin}[0]{c_{\mathrm{min}}}
\newcommand{\eps}[1]{\varepsilon_{#1}}
\newcommand{\imax}[0]{i_{\mathrm{max}}}
\newcommand{\vol}[1]{\mu\left(#1\right)}
\newcommand{\const}[0]{\Gamma}
\newcommand{\spacebest}[0]{\mathcal{X}_{\cbest{i}}}

\definecolor{myblue}{RGB}{158,202,225}
\definecolor{myred}{RGB}{252,146,114}
\definecolor{mygreen}{RGB}{161,217,155}
\definecolor{mypurple}{RGB}{190,174,212}

\begin{document}

\maketitle
\thispagestyle{empty}
\pagestyle{empty}

\begin{abstract}
We consider the problem of computing shortest paths in a dense motion-planning roadmap $\calG$.
We assume that~$n$, the number of vertices of $\calG$, is very large.
Thus, using any path-planning algorithm  that directly
searches~$\calG$, running in $O(V\textrm{log}V + E) \approx O(n^2)$ time, 
becomes unacceptably expensive.
We are therefore interested in anytime search to obtain successively shorter feasible paths and converge to the shortest path in~$\calG$.
Our key insight is to provide existing path-planning algorithms with a sequence of increasingly dense subgraphs of~$\calG$.
We study the space of all ($r$-disk) subgraphs of $\calG$.
We then formulate and present two densification strategies for traversing this space which exhibit complementary properties with respect to problem difficulty.
This inspires a third, hybrid strategy which has favourable properties regardless of problem difficulty.
This general approach is then demonstrated and analyzed  using the specific case where a low-dispersion deterministic sequence is used to generate the samples used for~$\calG$. 
Finally we empirically evaluate the performance of our strategies for random scenarios in $\mathbb{R}^{2}$ and $\mathbb{R}^{4}$ and on manipulation planning problems
for a 7 DOF robot arm,
and validate our analysis. 

\end{abstract}

\section{INTRODUCTION}
\label{sec:introduction}
Let $\calG$ be a motion-planning roadmap with $n$ vertices embedded in some configuration space (\Cs). 
We consider the problem of finding a shortest path between two vertices of $\calG$. 
Specifically, we are interested in settings, prevalent in motion planning, where testing if an edge of the graph is collision free 
or not is computationally expensive. We call such graphs Explicit graphs with Expensive Edge-Evaluation or E$^4$-graphs.
Moreover, we are interested in the case where~$n$ is very large,
and where the roadmap is dense, i.e. $|E| = O(n^2)$.
This makes any path-finding algorithm that directly searches~$\calG$, subsequently performing $O(n^2)$ edge evaluations, impractical.
We wish to obtain an approximation of the shortest path quickly and refine it as time permits.
We refer to this problem as \emph{anytime planning on large E$^4$-graphs}.

Our problem is motivated by previous work (\sref{sec:related_work}) on sampling-based motion-planning algorithms
that construct a \emph{fixed} roadmap as part of a preprocessing stage~\cite{KSLO96, BK00, SSH16}. 
These methods are used to efficiently approximate the structure of the \Cs.
When the size of the roadmap is large, even finding a solution, let alone an optimal one, becomes a non-trivial problem requiring specifically-tailored search algorithms~\cite{SSH16}. Our roadmap formulation departs from the PRM setting which chooses a connectivity radius $O(\text{log }n)$ that achieves asymptotic optimality. We are interested in dense, nearly-complete roadmaps that capture as much C-space connectivity information as possible, and probably have one or more paths that are strictly shorter than the optimal path for the standard PRM.

\begin{figure}
    \centering
    \begin{subfigure}[b]{0.25\columnwidth}
        \includegraphics[width=\textwidth]{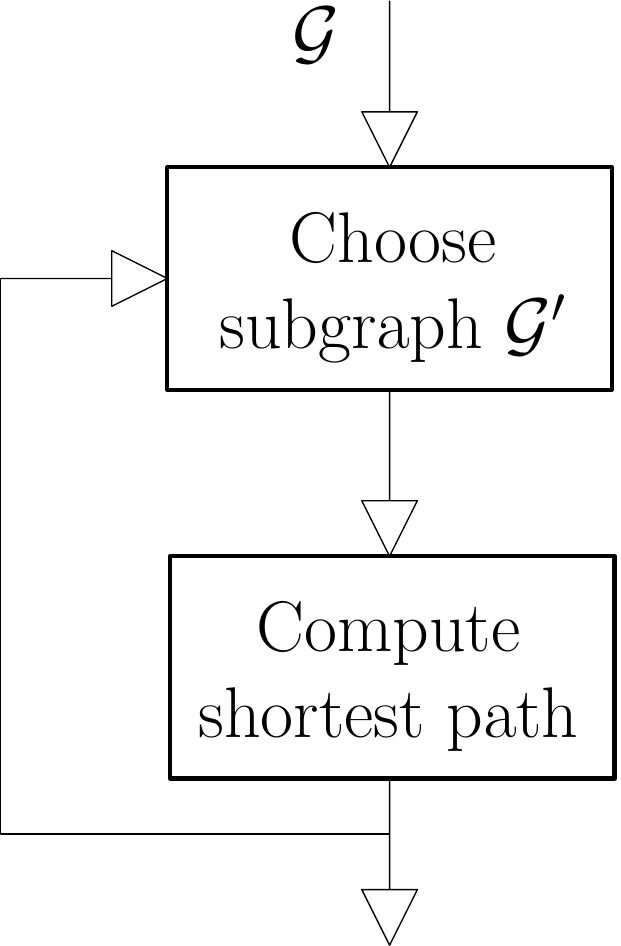}
        \vspace{15mm}
        \caption{}
        \label{fig:flow}
    \end{subfigure}
    \hfill
    \begin{subfigure}[b]{0.7\columnwidth}
        \includegraphics[width=\textwidth]{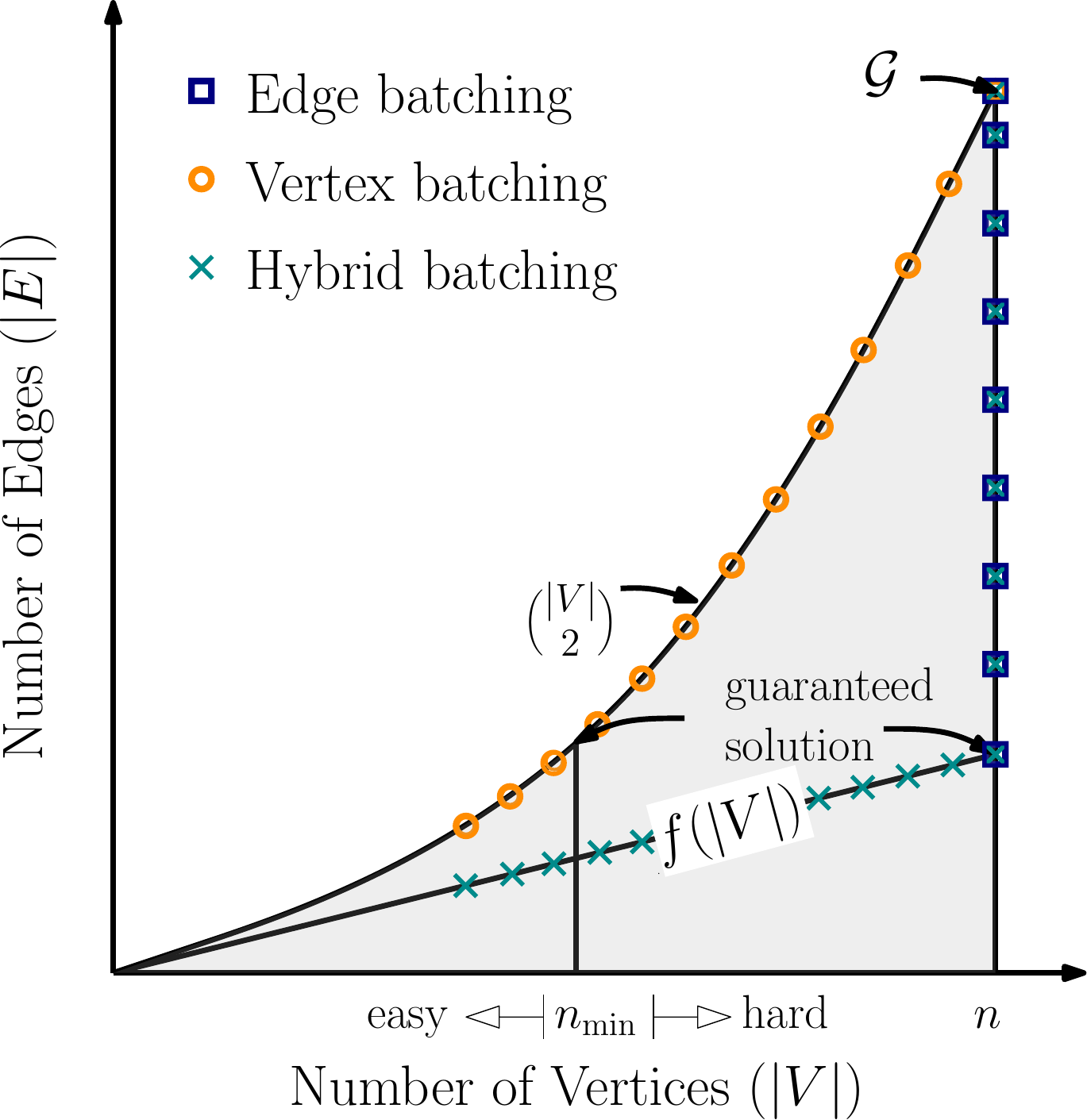}
        \caption{}
        \label{fig:space}
    \end{subfigure}

    \caption{
    (\subref{fig:flow})~Our meta algorithm leverages existing path-planning algorithms and provides them with a sequence of subgraphs.
        (\subref{fig:space})~To do so we consider densification strategies for traversing the space of $r$-disk subgraphs
    of the roadmap $\calG$. ,
    The $x$-axis and the $y$-axis represent the number of vertices and the number of edges (induced by $r$) of the subgraph, respectively. A particular subgraph is defined by a point in this space.
    \emph{Edge batching} searches over all samples and adds edges according to an increasing radius of connectivity. 
    \emph{Vertex batching} searches over complete subgraphs induced by progressively larger subsets of vertices.
    \emph{Hybrid batching} uses the minimal connection radius $f(|V|)$ to ensure connectivity until it reaches~$|V| = n$ and then proceeds like edge batching.
    The harder a problem, i.e.
    the lower the clearance between obstacles, the more vertices are needed by vertex and hybrid batching ($n_{\text{min}}$) to get their first
    feasible solution.}
    \label{fig:ve_batching}
\end{figure}

Our key insight for solving the anytime planning problem in large E$^4$-graphs is to provide existing
path-planning algorithms with a sequence of increasingly dense subgraphs of~$\calG$, using some  \emph{densification strategy}.
At each iteration, we run a shortest-path algorithm on the  current subgraph to obtain an increasingly tighter approximation of the true shortest path.
This favours using incremental search techniques that reuse information between calls.
We present a number of such strategies, and we address the question:

\begin{quote}
How does the densification strategy affect the time at which the first solution is found, and the quality of the solutions obtained?
\end{quote}

We focus on $r$-disk subgraphs of $\calG$, i.e. graphs defined by a specific set of vertices where every two vertices are connected if their mutual distance is at most $r$.
This induces a space of subgraphs (\figref{fig:ve_batching}) defined by the number of vertices and the connection radius (which, in turn, defines the number of edges).
We observe two natural ways to traverse this space.
The first is to define each subgraph over the entire set of vertices and incrementally add \emph{batches of edges} by increasing~$r$ 
(vertical line at $|V| = n$ in \figref{fig:ve_batching}).
Alternatively, we can incrementally add \emph{batches of vertices} and, at each iteration, consider the complete graph ($r = r_{\text{max}}$) defined over the current set of vertices
(parabolic arc $|E| = O(|V|^2)$ in \figref{fig:ve_batching}).
We call these variants \emph{edge batching} and \emph{vertex batching}, respectively.
Vertex batching and edge batching seem to be better suited for easy and hard problems, respectively,
as visualized and explained in Fig.~\ref{fig:viz2d_easy} and Fig.~\ref{fig:viz2d_hard}.
This analysis motivates our \emph{hybrid batching} strategy, which is more robust to problem difficulty.

Our main contribution is the formulation and analysis of various densification strategies to traverse the space of subgraphs of~$\calG$ (\sref{sec:approach}).
We analyse the specific case where the vertices of~$\calG$ are obtained from a low-dispersion deterministic sequence  (\sref{sec:instantiation}). Specifically, we describe the structure of the space of subgraphs and demonstrate the tradeoff between effort and bounded suboptimality for our densification strategies. 
Furthermore, we explain how this tradeoff varies with problem difficulty, which is measured in terms of the clearance of the shortest path in~$\calG$. 

We discuss implementation decisions and parameters that allow us to efficiently use our strategies in dense~E$^4$ graphs in \sref{sec:implementation}.
We then empirically validate our analysis on several random scenarios in $\mathbb{R}^{2}$ and $\mathbb{R}^{4}$ and on manipulation planning problems
for a 7 DOF robot arm (\sref{sec:experiments}). 
Finally, we discuss directions of future work (\sref{sec:conclusion}).


\begin{figure}
    \begin{subfigure}[b]{0.325\columnwidth}
    \centering
        \frame{\includegraphics[width=\textwidth]{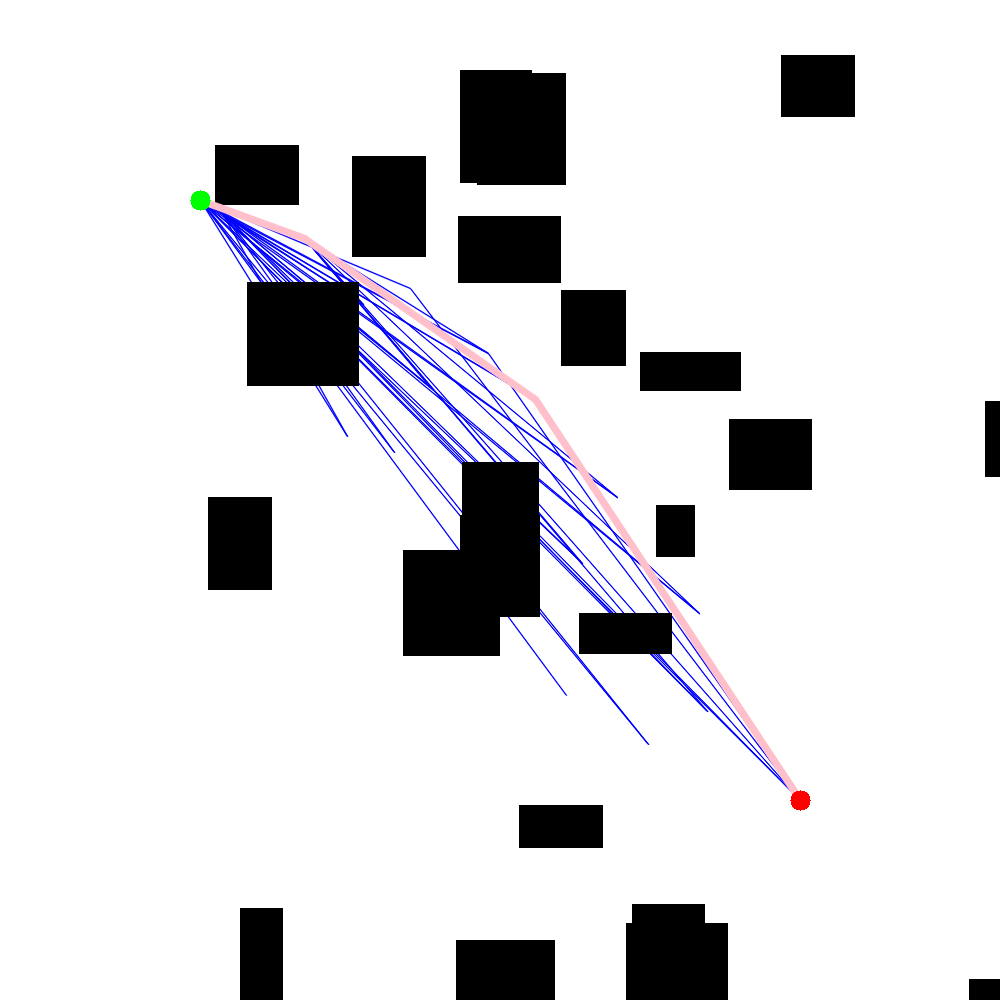}}
        \caption{40 checks}
        \label{fig:viz2d_vertex_easy1}
    \end{subfigure}
    \begin{subfigure}[b]{0.325\columnwidth}
    \centering
        \frame{\includegraphics[width=\textwidth]{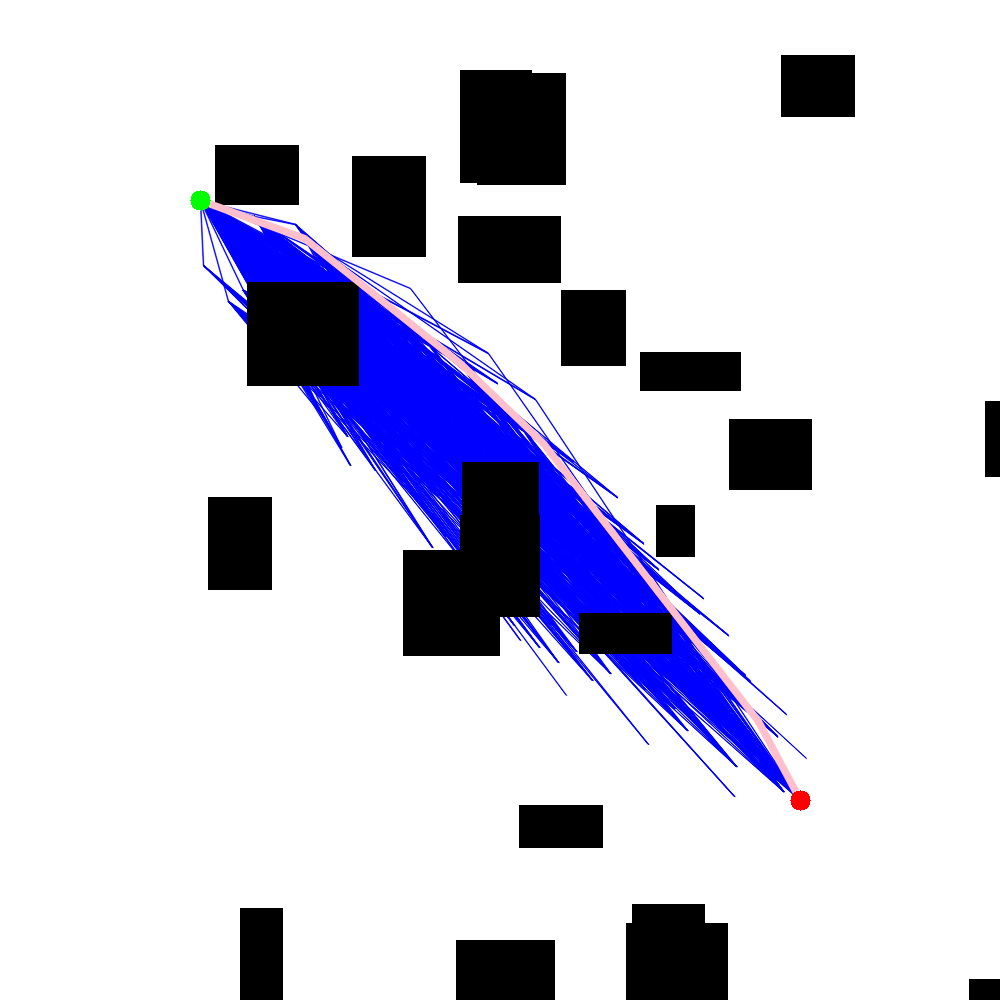}}
        \caption{953 checks}
        \label{fig:viz2d_vertex_easy2}
    \end{subfigure}
    \begin{subfigure}[b]{0.325\columnwidth}
    \centering
        \frame{\includegraphics[width=\textwidth]{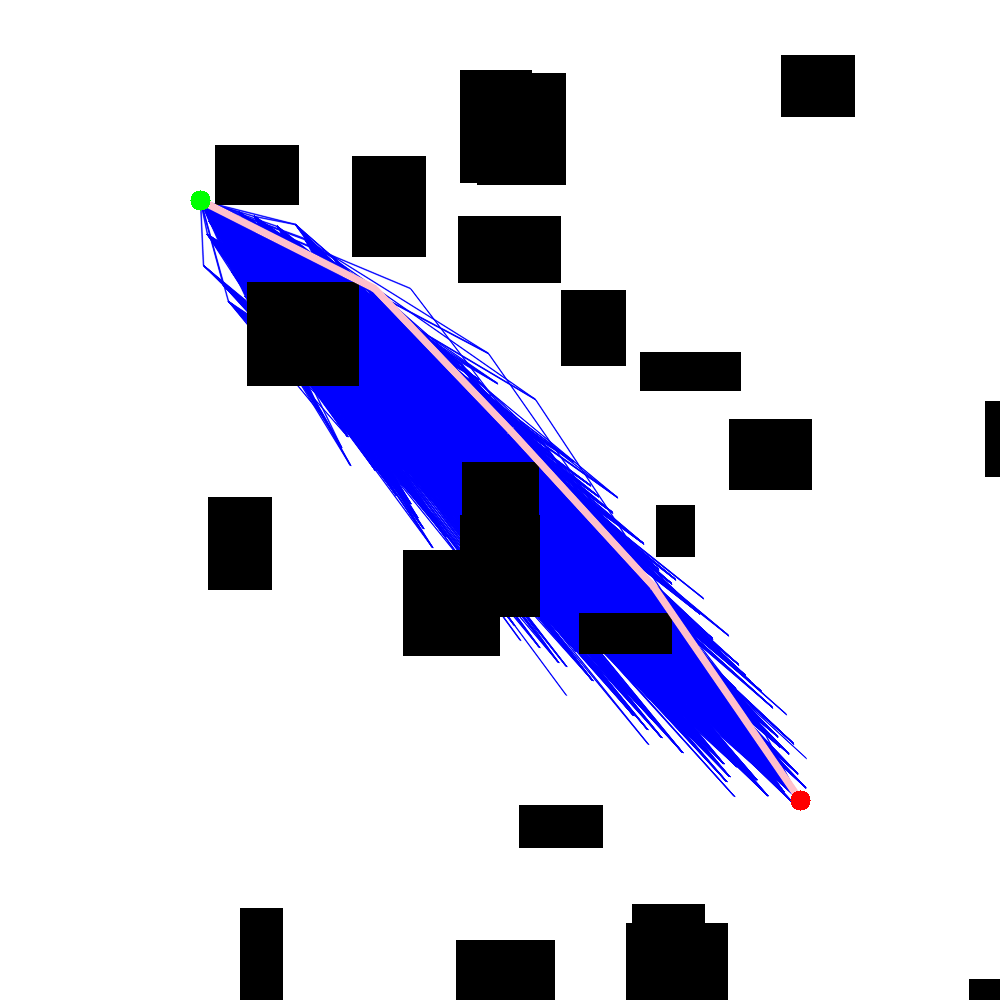}}
        \caption{6,310 checks}
        \label{fig:viz2d_vertex_easy3}
    \end{subfigure}
    \begin{subfigure}[b]{0.325\columnwidth}
    \centering
        \frame{\includegraphics[width=\textwidth]{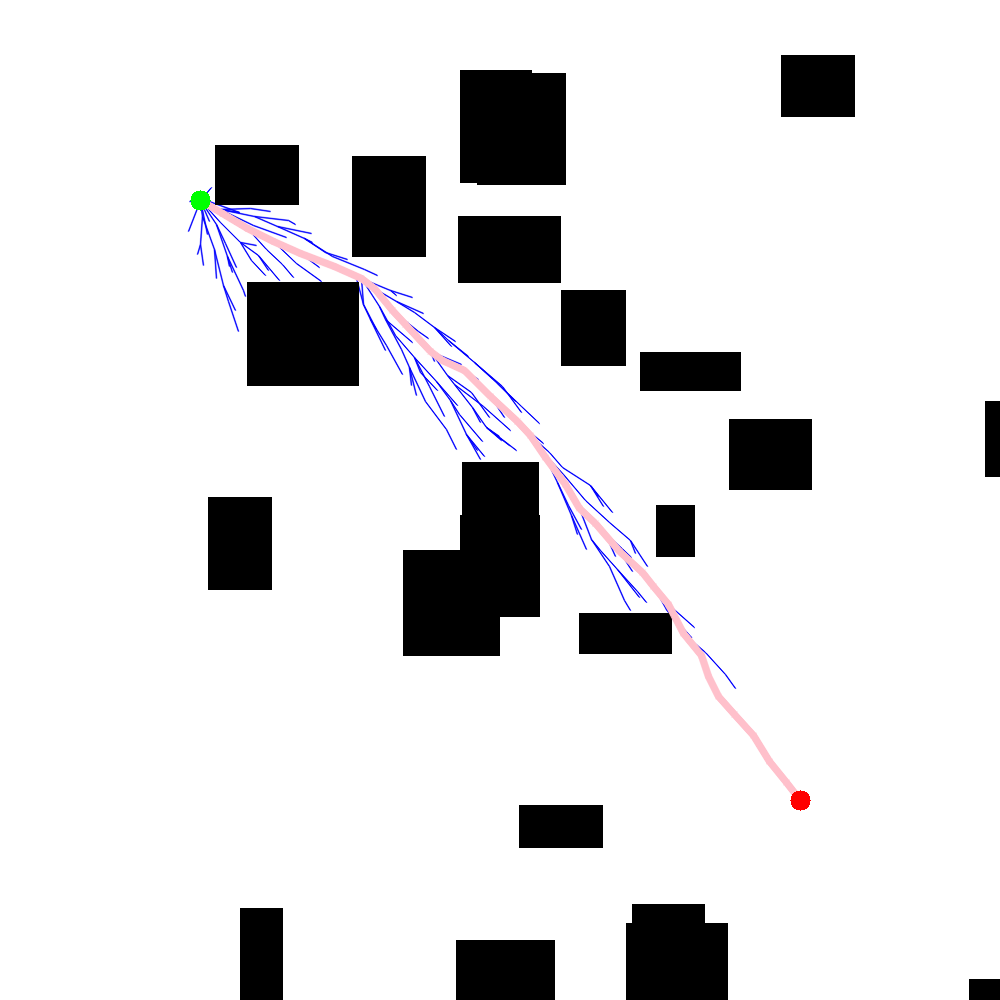}}
        \caption{206 checks}
        \label{fig:viz2d_edge_easy1}
    \end{subfigure}
    \begin{subfigure}[b]{0.325\columnwidth}
    \centering
        \frame{\includegraphics[width=\textwidth]{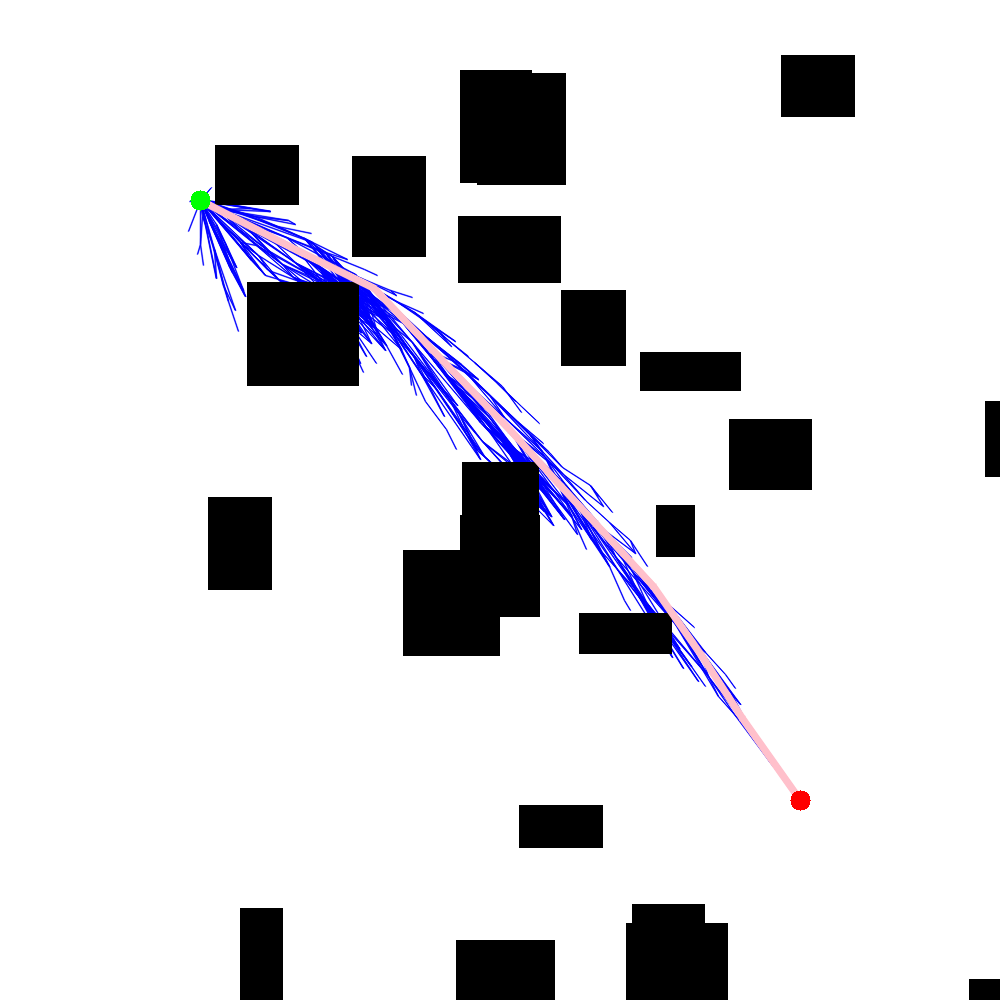}}
        \caption{676 checks}
        \label{fig:viz2d_edge_easy2}
    \end{subfigure}
    \begin{subfigure}[b]{0.325\columnwidth}
    \centering
        \frame{\includegraphics[width=\textwidth]{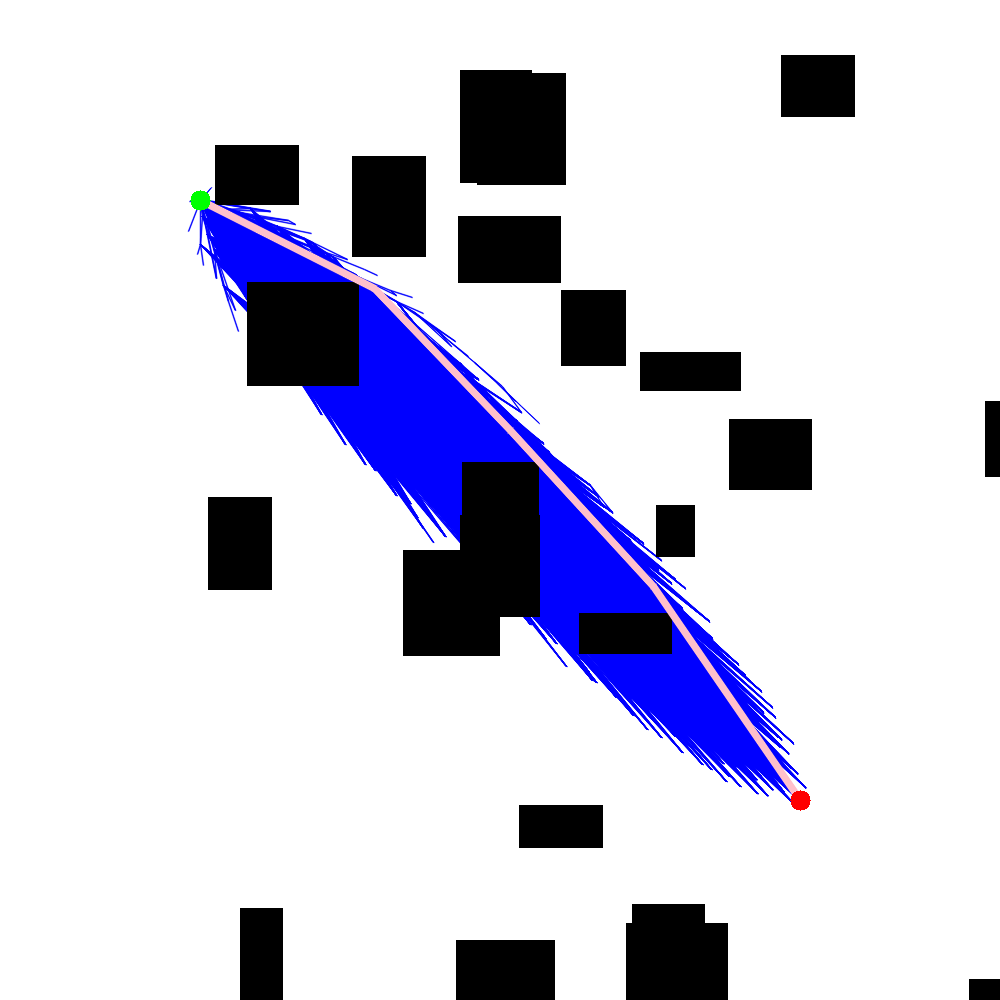}}
        \caption{15,099 checks}
        \label{fig:viz2d_edge_easy3}
    \end{subfigure}

    \caption{\ A visualization of an easy $\mathbb{R}^{2}$ problem, with higher average obstacle gaps. 
    The initial and final paths are shown,
    as well as an intermediate path. Vertex batching (\subref{fig:viz2d_vertex_easy1}-\subref{fig:viz2d_vertex_easy3}) finds the 
    first solution quickly with a sparse set of initial samples. Additional heuristics hereafter help it converge to the optimum with fewer edge evaluations than edge batching (\subref{fig:viz2d_edge_easy1}-\subref{fig:viz2d_edge_easy3}). Note that the coverage of collision checks only appears similar at the end due to resolution
    limits for visualization.}

    \label{fig:viz2d_easy}
\end{figure}

\section{RELATED WORK}
\label{sec:related_work}

\subsection{Sampling-based motion planning}
Sampling-based planning approaches build a graph, or a roadmap, in the \Cs, where vertices are configurations and edges are local paths connecting configurations. 
A path is then found by traversing this roadmap while checking if the vertices and edges are collision free.
Initial algorithms such as  PRM~\cite{KSLO96} and RRT~\cite{LK99} were concerned with finding \emph{a feasible} solution. 
However, in recent years, there has been growing interest in finding high-quality solutions.
Karaman and Frazzoli~\cite{KF11} introduced variants of the PRM and RRT algorithms, called PRM* and RRT*, respectively and proved that, asymptotically,  the solution obtained by these algorithms converges to the optimal solution. 
However, the running times of these algorithms are often significantly higher than their non-optimal counterparts.
Thus, subsequent algorithms have been suggested to increase the rate of convergence to high-quality solutions.
They use different approaches such as 
lazy computation~\cite{BK00,janson2015fast,salzman2015asymptotically}, 
informed sampling~\cite{GSB14}, 
pruning vertices~\cite{GSB15},
relaxing optimality~\cite{SH16}, exploiting local information~\cite{choudhury2016regionally} 
and
lifelong planning together with heuristics~\cite{CPL14}.
In this work we employ several such techniques in order to speed up the convergence rate of our algorithms.

\begin{figure}
    \begin{subfigure}[b]{0.325\columnwidth}
    \centering
        \frame{\includegraphics[width=\textwidth]{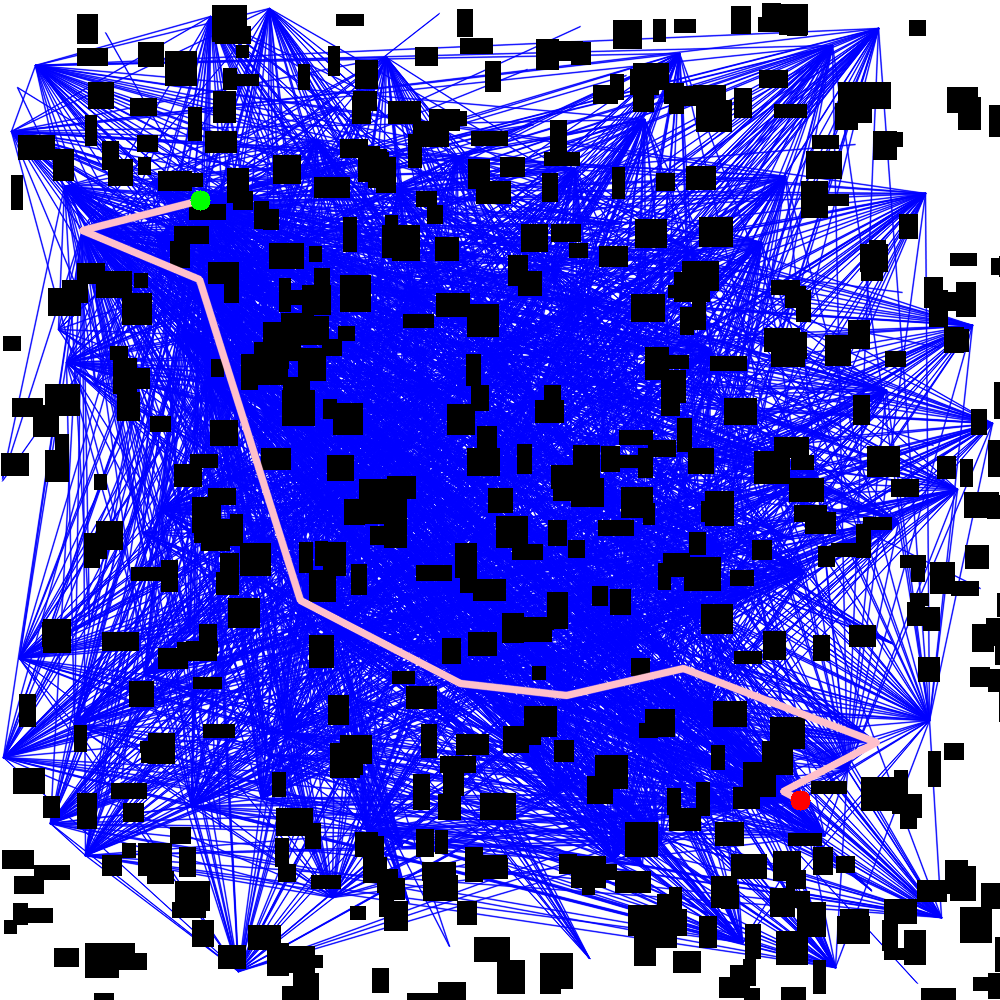}}
        \caption{2,573 checks}
        \label{fig:viz2d_vertex_hard1}
    \end{subfigure}
    \begin{subfigure}[b]{0.325\columnwidth}
    \centering
        \frame{\includegraphics[width=\textwidth]{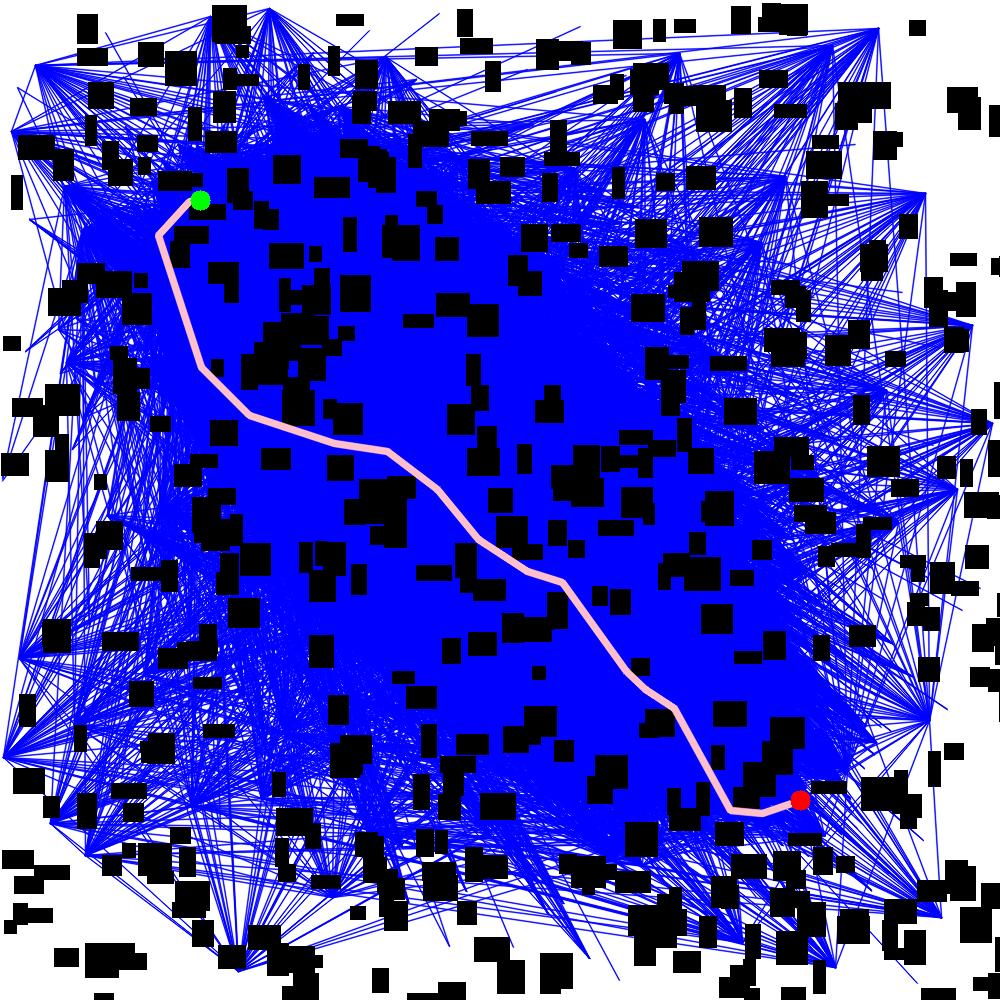}}
        \caption{61,506 checks}
        \label{fig:viz2d_vertex_hard2}
    \end{subfigure}
    \begin{subfigure}[b]{0.325\columnwidth}
    \centering
        \frame{\includegraphics[width=\textwidth]{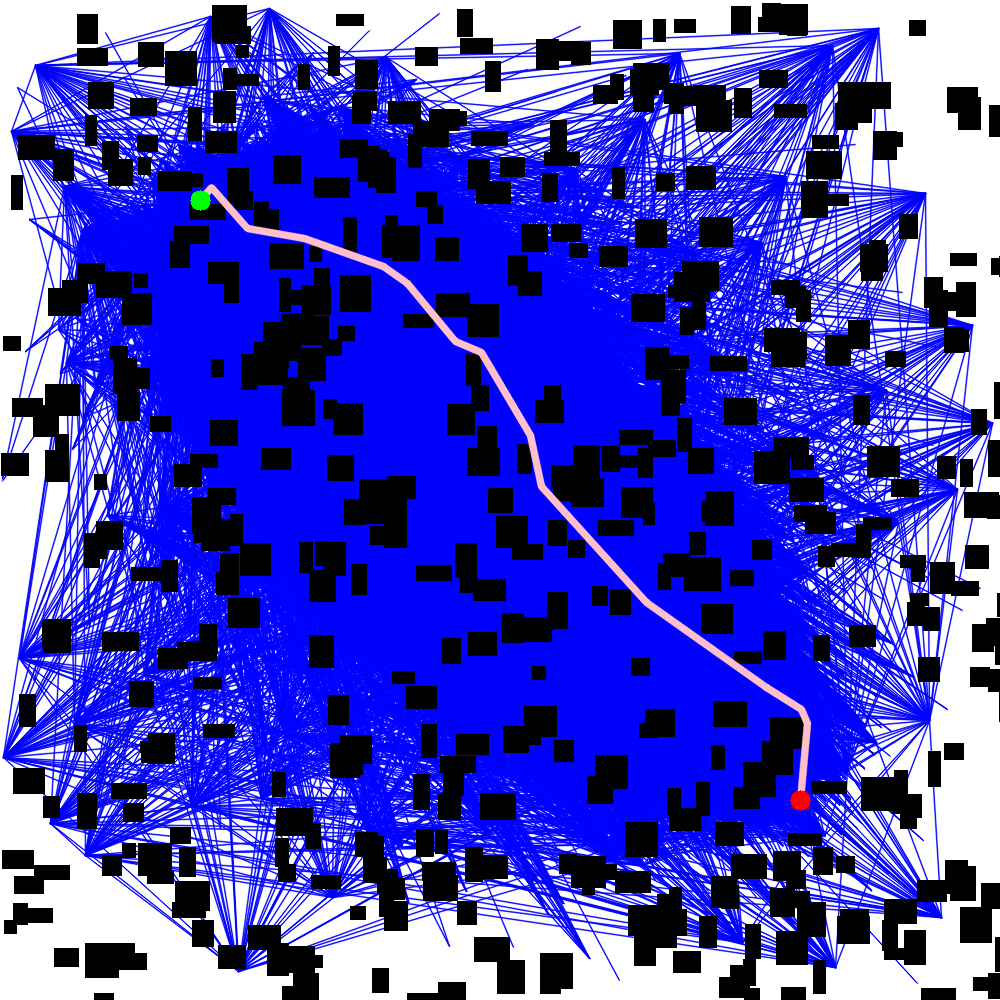}}
        \caption{164,504 checks}
        \label{fig:viz2d_vertex_hard3}
    \end{subfigure}
    \begin{subfigure}[b]{0.325\columnwidth}
    \centering
        \frame{\includegraphics[width=\textwidth]{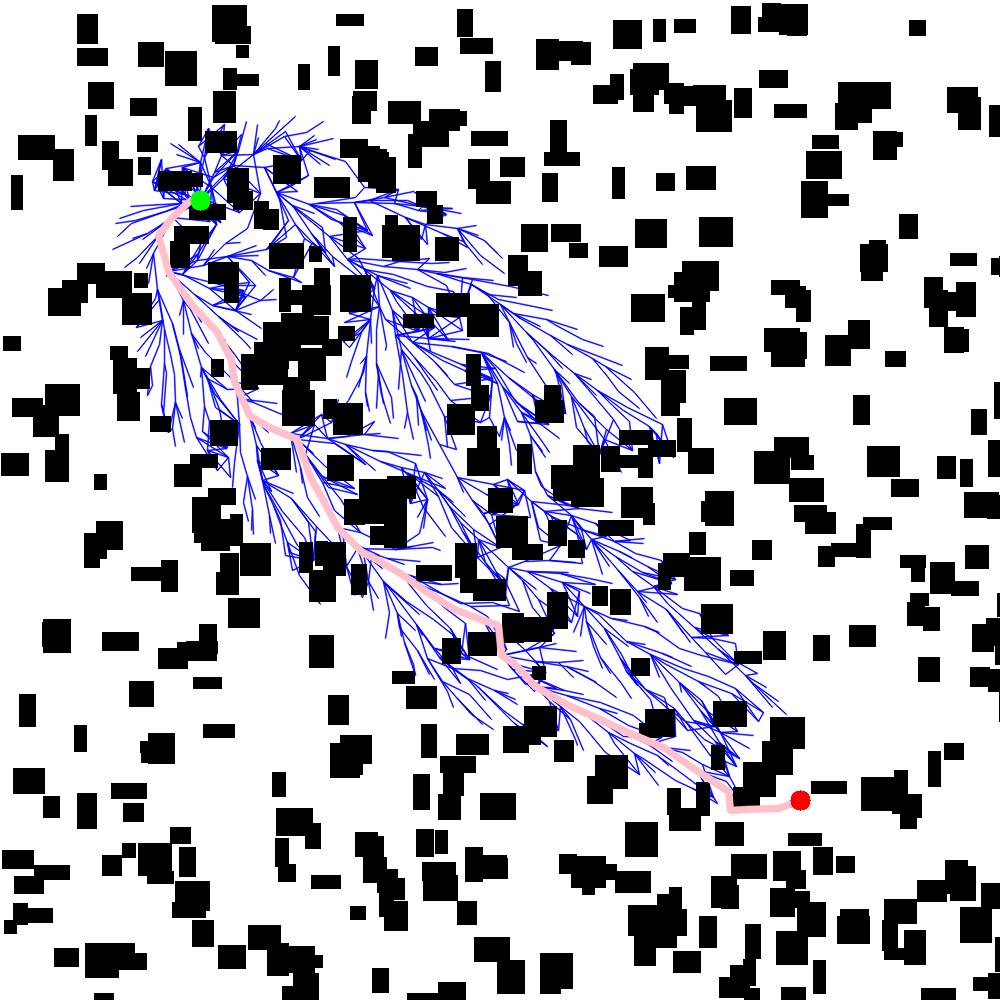}}
        \caption{1390 checks}
        \label{fig:viz2d_edge_hard1}
    \end{subfigure}
    \begin{subfigure}[b]{0.325\columnwidth}
    \centering
        \frame{\includegraphics[width=\textwidth]{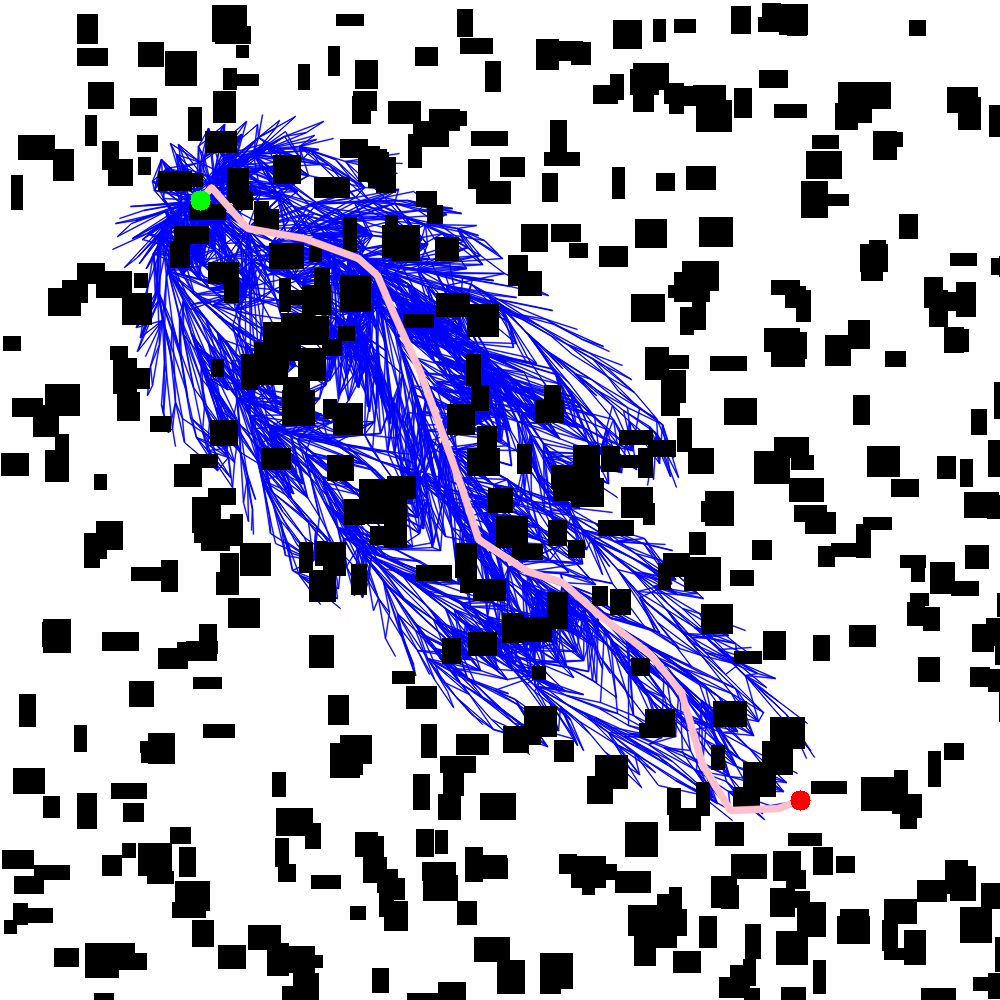}}
        \caption{4,687 checks}
        \label{fig:viz2d_edge_hard2}
    \end{subfigure}
    \begin{subfigure}[b]{0.325\columnwidth}
    \centering
        \frame{\includegraphics[width=\textwidth]{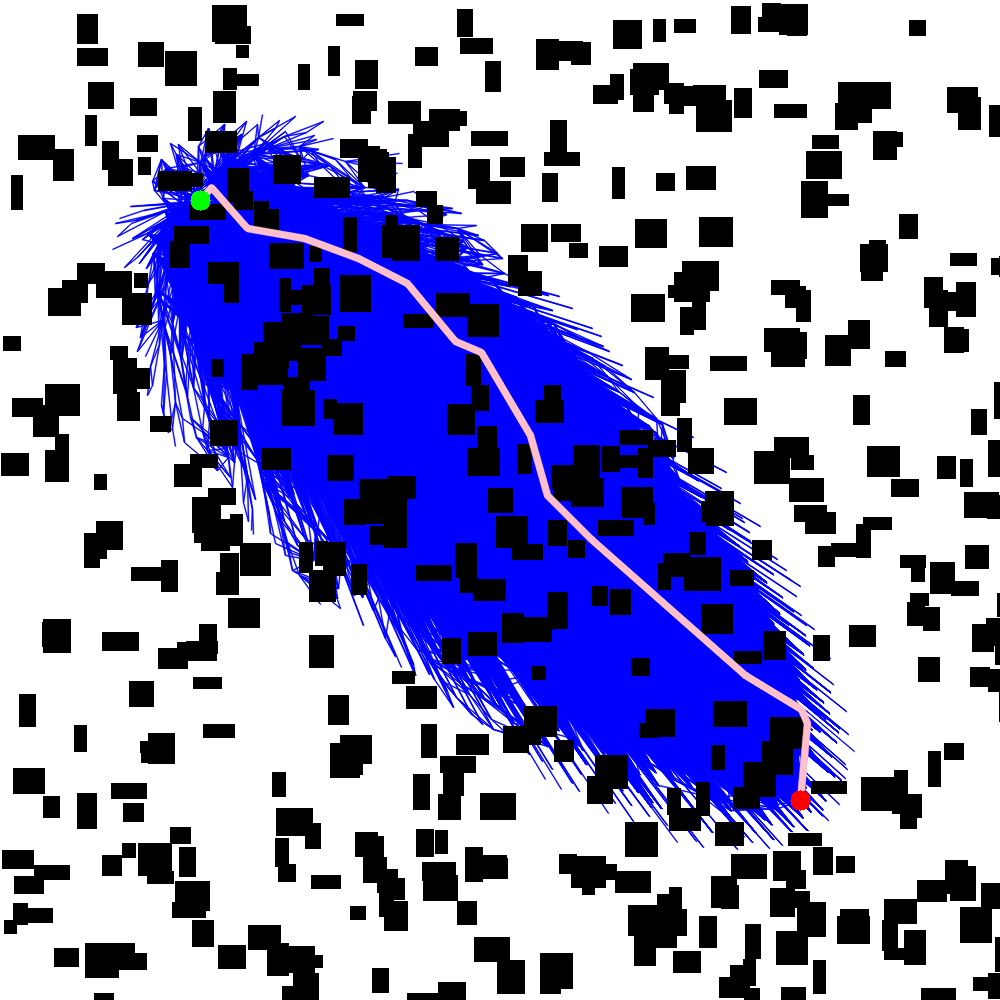}}
        \caption{78,546 checks}
        \label{fig:viz2d_edge_hard3}
    \end{subfigure}

    \caption{A visualization of a hard $\mathbb{R}^{2}$ problem with $10\times$ more obstacles and lower average obstacle gaps. 
    The same set of samples $\calS$ is used as in 
    Fig.~\ref{fig:viz2d_easy}. Because the problem is more difficult, both vertex and edge batching require
    more edge evaluations for finding feasible solutions and converging to the optimum.
    In particular, vertex batching (\subref{fig:viz2d_vertex_hard1}-\subref{fig:viz2d_vertex_hard3})
    requires multiple iterations to find its first solution,
    while edge batching (\subref{fig:viz2d_edge_hard1}-\subref{fig:viz2d_edge_hard3})
    still does so on its first search, albeit with more collision checks than for the easy problem.}
    \label{fig:viz2d_hard}
\end{figure}

\subsection{Finite-time properties of sampling-based algorithms}
\label{subsec:finite}
Extensive analysis has been done on \emph{asymptotic} properties of sampling-based algorithms,
i.e. properties such as connectivity and optimality when the number of samples tends to infinity~\cite{KKL98, KF11}.

We are interested in bounding the quality of a solution obtained using a \emph{fixed} roadmap for a finite number of samples.
When the samples are generated from a \emph{deterministic} sequence, Janson et. al.~\cite[Thm2]{JIP15} give a closed-form solution bounding the quality of the solution of a PRM whose roadmap is an $r$-disk graph. The bound is a function of $r$, the number of vertices $n$ and the \emph{dispersion} of the set of points used. (See \sref{sec:problem_formulation} for an exact definition of dispersion and for the exact bound given by Janson et. al.).

Dobson et. al.~\cite{DMB15} provide similar bounds when randomly sampled i.i.d points are used.
Specifically, they consider a PRM whose roadmap is an $r$-disk graph for a \emph{specific} radius $r = c \cdot \left(\log n / n\right)^{1/d}$ where $n$ is the number of points, $d$ is the dimension and $c$ is some constant.
They then give a bound on the probability that the quality of the solution will be larger than a given threshold.

\subsection{Efficient path-planning algorithms}
We are interested in path-planning algorithms that attempt to reduce the amount of 
computationally expensive edge expansions performed in a search.
This is typically done using heuristics such 
as for A*~\cite{HNR68}, for
Iterative Deepening A*~\cite{K85} and for
Lazy Weighted A*~\cite{CPL14}.
Some of these algorithms, such as Lifelong Planning A*~\cite{koenig2004lifelong} allow recomputing the shortest path in an efficient manner when the graph undergoes changes.
\emph{Anytime} variants of A* such as
Anytime Repairing A*~\cite{LGT03} 
and
Anytime Nonparametric A*~\cite{BSHG11}
efficiently run a succession of A* searches, each with an inflated heuristic. 
This potentially obtains a fast approximation and refines its quality as time permits.
However, there is no formal guarantee that these approaches will decrease search time and they may still search all edges of a given graph~\cite{WR12}.
For a unifying formalism of such algorithms relevant to E$^4$ graphs and additional references, see~\cite{dellin2016unifying}.

\section{NOTATION, PROBLEM FORMULATION AND MATHEMATICAL BACKGROUND }
\label{sec:problem_formulation}
We provide standard notation and define our problem concretely.
We then provide necessary mathematical background about the \emph{dispersion} of a set of points.

\subsection{Notation and problem formulation}
Let~$\calX$ denote a $d$-dimensional  \Cs, $\Cfree$ the collision-free portion of $\calX$,  $\Cobs = \calX \setminus \Cfree$ 
its complement
and
let $\rho: \calX \times \calX \rightarrow \R$ be some distance metric.
For simplicity, we assume that $\calX = [0,1]^d$ and that $\rho$ is the Euclidean norm.
Let  $\calS = \{ s_1, \ldots, s_n\}$ be some sequence of points 
where $s_\ell \in \calX$
and denote by $\calS(\ell)$ the first $\ell$ elements of $\calS$. 
We define the $r$-disk graph $\calG(\ell, r) = (V_{\ell},E_{\ell,r})$ 
where
$V_{\ell} = \calS(\ell)$, 
$E_{\ell,r} = \{(u,v) \ | \ u,v \in V_\ell \text{ and } \rho(u,v) \leq r \}$
and each edge $(u,v)$ has a weight $w(u,v) = \rho(u,v)$.
See~\cite{KF11, SSH16c} for various properties of such graphs in the context of motion planning.
Finally, set~$\calG = \calG(n, \sqrt{d})$, namely, the complete\footnote{Using a radius of $\sqrt{d}$ ensures that every two points will be connected due to the assumption that $\calX = [0,1]^d$ and that $\rho$ is Euclidean.} graph defined over $\calS$. 

For ease of analysis we assume that the  roadmap is complete, but our densification strategies and analysis can be extended to \emph{dense} roadmaps that are not complete.
Furthermore, our definition assumes that $\calG$ is embedded in the \Cs. 
Thus, we will use the terms \emph{vertices} and \emph{configurations} as well as \emph{edges} and \emph{paths} in \Cs interchangeably.

A query $\calQ$ is a scenario with start and target configurations. 
Let the start and target configurations be $s_1$ and $s_2$, respectively. 
The obstacles induce a mapping $\calM: \calX \rightarrow \{\Cfree,\Cobs\}$ called a \emph{collision detector} which checks if a configuration or edge is collision-free or not.
Typically, edges are checked by densely sampling along the edge, hence the term \emph{expensive edge evaluation}. 
A feasible path is denoted by $\gamma: [0,1] \rightarrow \Cfree$ where $\gamma[0] = s_1$ and $\gamma[1] = s_2$.
Slightly abusing this notation, 
set $\gamma(\calG(\ell, r))$ to be the shortest collision-free path from $s_1$ to $s_2$ that can be computed in $\calG(\ell, r)$,
its clearance as $\delta(\calG(\ell, r))$ and denote by
$\gamma^* = \gamma(\calG)$ and
$\delta^* = \delta(\calG)$
the shortest path and its clearance that can be computed in $\calG$, respectively.
Note that a path has clearance $\delta$ if every point on the path is at a distance of at least
$\delta$ away from every obstacle.

Our problem calls for finding a sequence of increasingly shorter
feasible paths  $\gamma_0$, $\gamma_1 \ldots$ in $\calG$,  converging
to $\gamma^{*}$.
We assume that $n = |\calS|$ is sufficiently large, and the roadmap covers the space well enough so that for any reasonable set of obstacles, there are multiple feasible paths to be obtained between start and goal. Therefore, we do not consider a case where the entire roadmap is invalidated by obstacles. The large value of $n$ makes any path-finding algorithm that directly searches~$\calG$, thereby performing $O(n^2)$ calls to
the collision-detector, too time-consuming to be practical.

\subsection{Dispersion}
The \emph{dispersion} $D_n(\calS)$ of a sequence $\calS$ is defined as 
$
D_n(\calS) = \sup_{x \in \calX} \min_{s \in \calS} \rho(x, s) 
$.
Intuitively, it can be thought of as the radius of the largest empty ball (by some metric)
that can be drawn around any point in the space $\calX$ without intersecting any point
of $\calS$. A lower dispersion implies a better \emph{coverage} of the space by the points in $\calS$.
When $\calX$ is the $d$-dimensional 
Euclidean space and $\rho$ is the Euclidean distance, deterministic sequences  with dispersion of order $O(n^{-1/d})$ exist.
A simple example is a set of points lying on grid or a lattice.

Other low-dispersion deterministic sequences exist which also  have low discrepancy, i.e. they appear to be random for many purposes.
One such example is the \emph{Halton sequence}~\cite{H60}. We will use them extensively for our analysis because
they have been studied in the context of deterministic motion planning~\cite{JIP15,BLOY01}. 
For Halton sequences, tight bounds on dispersion exist. Specifically, $D_n(\calS) \leq p_d \cdot n^{-1/d}$ where $p_d \approx d \text{log } d$ is the $d^{th}$ prime number.
Subsequently in this paper, we will use $D_n$ (and not~$D_n(\calS)$) to denote the dispersion of the first $n$ points of~$\calS$.

Janson et. al. bound the length of the shortest path computed over an $r$-disk roadmap constructed using a low-dispersion deterministic sequence~\cite[Thm2]{JIP15}.
Specifically, given start and target vertices, consider all paths $\Gamma$ connecting them which 
have $\delta$-clearance for some $\delta$. Set $\delta_{\text{max}}$ to be the maximal clearance over all such $\delta$.
If $\delta_{\text{max}} > 0$, then
for all $0<\delta \leq \delta_{\text{max}}$ 
set $c^*(\delta)$ to be the cost of the shortest path in $\Gamma$ with $\delta$-clearance.
Let $c(\ell,r)$ be the length of the path returned by a shortest-path algorithm on $\calG(\ell,r)$ with $\calS(\ell)$ having dispersion $D_\ell$.
For $2D_\ell < r < \delta$, we have that
\begin{equation}
\label{eq:dispersion_suboptimality}
c(\ell,r)
\leq 
\left( 
    1  + \frac{2D_\ell}{r - 2D_\ell}
\right) \cdot c^*(\delta).
\end{equation}

Notably, for $n$ random i.i.d. points, 
the lower bound on the dispersion is $O\left( (\log n / n )^{1/d}\right)$~\cite{N92} which is strictly larger than for deterministic samples.
\\

For domains other than the unit hypercube, the insights from the analysis will generally hold. However, the dispersion bounds may become far more complicated depending on the domain, and the distance metric would need to be scaled accordingly. This may result in the quantitative bounds being difficult to deduce analytically.

\section{Approach}
\label{sec:approach}
We now discuss our general approach of searching over the space of all ($r$-disk) subgraphs of $\calG$.
We start by characterizing the boundaries and different regions of this space. 
Subsequently, we introduce two densification strategies---edge batching and vertex batching.
As we will see, these two are complementary in nature, which motivates our third strategy, which we call hybrid batching.

\subsection{The space of subgraphs}

To perform an anytime search over $\calG$, 
we iteratively search a sequence of graphs
$\calG_0(n_0, r_0) \subseteq \calG_1(n_1, r_1)  \subseteq  \ldots \subseteq \calG_m(n_m, r_m) = \calG$.
If no feasible path exists in the subgraph, we move on to the next subgraph in the sequence, which is more likely to have a feasible path.

\begin{figure}
    \centering
    \includegraphics[width=0.8\columnwidth]{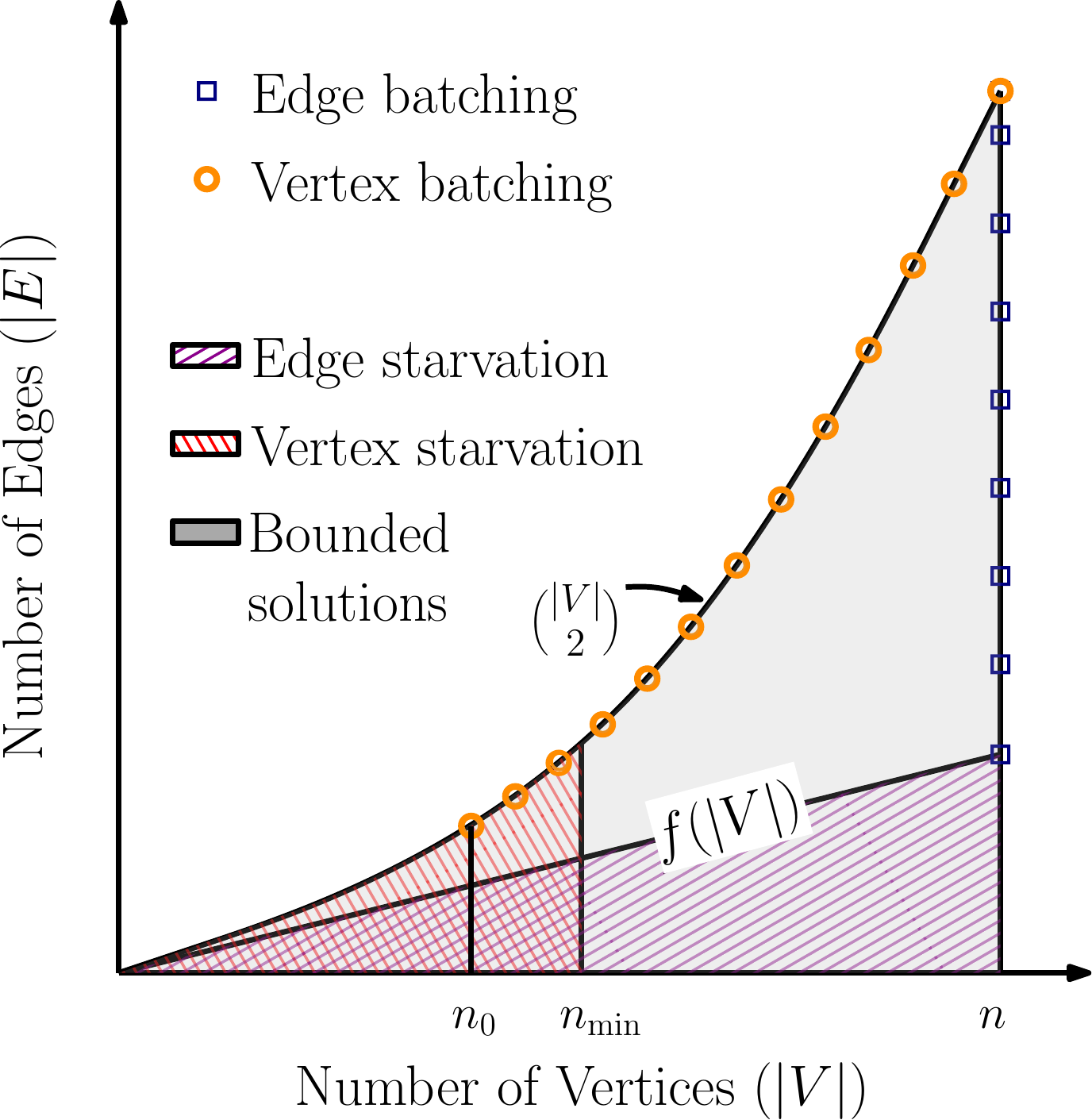}
    \caption{ Regions of interest for the space of subgraphs. \emph{Vertex Starvation} happens in the region
    with too few vertices to ensure a solution, even for a fully connected subgraph. \emph{Edge Starvation}
    happens in the region where the radius $r$ is too low to guarantee connectivity.}
    \label{fig:ve_starvation}
\end{figure}

We use an incremental path-planning algorithm that allows us to efficiently recompute shortest paths.
Our problem setting of increasingly dense subgraphs is particularly amenable to such algorithms.
However, any alternative shortest-path algorithm may be used. 
We emphasize again that we focus on the meta-algorithm of choosing which subgraphs to search. Further details on the implementation of these approaches are provided in \sref{sec:implementation}.

\figref{fig:ve_starvation} depicts the set of possible graphs~$\calG(\ell, r)$ for all choices of  $0 < \ell \leq n$ and $0 < r \leq \sqrt{d}$.
Specifically, the graph depicts $|E_{\ell, r}|$ as a function of $|V_\ell|$.
We discuss \figref{fig:ve_starvation} in detail to motivate our approach for solving the problem of anytime planning in large E$^4$-graphs and the specific sequence of subgraphs we use.
First, consider the curves that define the boundary of all possible graphs:
The vertical line $|V| = n$ corresponds to subgraphs defined over the entire set of vertices, where batches of edges are added as $r$ increases.
The parabolic arc $|E| = |V|\cdot(|V|-1)/2$, corresponds to complete subgraphs defined over increasingly larger sets of vertices.

Recall that we wish to approximate the shortest path $\gamma^*$ which has some minimal clearance $\delta^*$.
Given a specific graph, to ensure that a path that approximates $\gamma^*$ is found, 
two conditions should be met:
(i)~The graph includes some minimal number $n_{\text{min}}$ of vertices. The exact value of~$n_{\text{min}}$ will be a function of the dispersion $D_{n_{\text{min}}}$ of the sequence $\calS$ and the clearance $\delta^*$.
(ii)~A minimal connection radius $r_0$ is used to ensure that the graph is connected.
Its value will depend on the sequence $\calS$ (and not on~$\delta^*$).

Requirement (i) induces a vertical line at $|V| = n_{\text{min}}$.
Any point to the left of this line corresponds to a graph with too few vertices to prove any guarantee that a solution will be found. We call this the \emph{vertex-starvation} region.
Requirement~(ii) induces a curve $f(|V|)$ such that any point below this curve corresponds to a graph which may be disconnected. 
We call this the \emph{edge-starvation} region.
The exact form of the curve depends on the sequence~$\calS$ that is used. 
The specific value of $n_{\min}$ and the form of $f(|V|)$  when Halton sequences are used is provided in~\sref{sec:instantiation}.

Any point outside the starvation regions represents a graph~$\calG(\ell,r)$ such that the length of $\gamma(\calG(\ell,r))$ may be bounded.
For a discussion on specific bounds, see \sref{subsec:finite}.
For a visualization of the different regions, see~\figref{fig:ve_starvation}.

\subsection{Edge and vertex batching}
Our goal is to search increasingly dense subgraphs of~$\calG$. 
This corresponds to a sequence of points on the space of subgraphs (\figref{fig:ve_starvation}) that ends at the upper right corner of the space. 
Two natural strategies emerge from this.
We defer the discussion on the choice of parameters used for each strategy to~\sref{sec:implementation}.

%
\subsubsection{Edge batching} 
All subgraphs include the complete set of vertices $\calS$ and the edges are incrementally added
via an increasing connection radius.
Specifically, $\forall i~n_i = n$ and $r_{i} = \eta_{e} r_{i-1}$ where
$\eta_{e} > 1$ and $r_{0}$ is some small initial radius. 
Here, we choose
$r_0 = O(f(n))$, where $f$ is the \emph{edge-starvation} boundary
curve defined previously.
Using \figref{fig:ve_starvation}, this induces a sequence of points along the vertical line at $|V| = n$ starting from $|E| = O(n^2 r_0^d)$ and ending at $|E| = O(n^2)$.
\subsubsection{Vertex batching} 
In this variant, all subgraphs are complete graphs defined over increasing subsets of the complete set of vertices $\calS$.
Specifically $\forall i~r_i = r_{\text{max}} = \sqrt{d}$,
$n_i = \eta_{v} n_{i-1}$ where $\eta_{v} > 1$ and the base term $n_0$ is some small number of vertices. Because we have no priors about the obstacle density or distribution, the chosen $n_0$ is a constant and does not vary due to $n$ or due to the volume of $\Cobs$.
Using \figref{fig:ve_starvation}, this induces a sequence of points along the parabolic arc $|E| = |V|\cdot (|V|-1)/2$
starting from $|V| = n_0$ and ending at $|V| = n$. The vertices are chosen in the same
order with which they are generated by~$\calS$. So, $\calG_0$ has the first $n_0$ samples of
$\calS$, and so on.

\begin{figure}
    \centering
    \includegraphics[width=0.8\columnwidth]{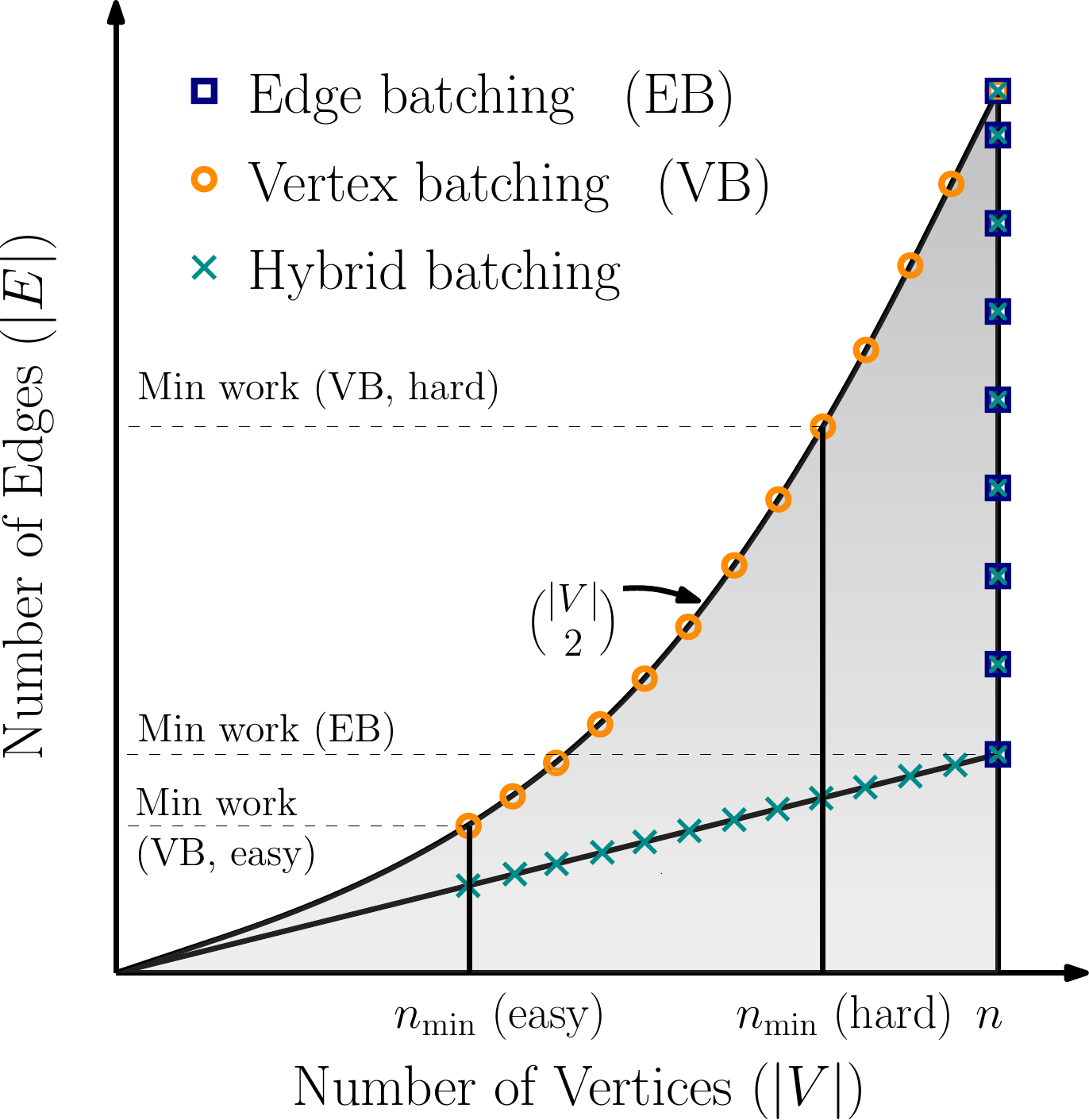}
    \caption{\ Visualization of the work required by our densification strategies as a function of the problem's hardness. Here work is measured as the number of edges evaluated.
    This is visualized using the gradient shading where light gray (resp. dark grey) depicts a small (resp. large) amount of work.
    Assuming $n > n_{\min}$, the amount of work required by edge batching remains the same regardless of problem difficulty.
    For vertex batching the amount of work required depends on the hardness of the problem.
    Here we visualize an easy and a hard problem using~$n_{\min}$ (easy) and $n_{\min}$ (hard), respectively.}
    \label{fig:ve_comparitive}
\end{figure}

Intuitively, the relative performance of these densification
strategies depends on problem hardness. 
We use the clearance of the shortest path, $\delta^{*}$, to represent the hardness of the problem.
This, in turn, defines $n_{\min}$ which bounds the vertex-starvation region.
Specifically we say that a problem is easy (resp. hard) when
$\delta^* \approx \sqrt{d}$ 
(resp. $\delta^* \approx \Omega(D_n(\calS))$).
For easy problems, with larger gaps between obstacles,
vertex batching can find a solution quickly with fewer samples and long edges, thereby restricting the work done
for future searches.
In contrast, assuming that $n > n_{\min}$, edge batching will find a solution on the first iteration but 
the time to do so may be far greater than for vertex batching because the number of samples
is so large.
For hard problems
vertex batching may require multiple iterations until the number of samples it uses is large enough and it is out of the vertex-starvation region.
Each of these searches would exhaust the fully connected subgraph before terminating. 
This cumulative effort is expected to exceed that required by edge batching for the same problem, which is expected to find a feasible
albeit sub-optimal path on the first search.
A visual depiction of this intuition is given in \figref{fig:ve_comparitive}. 

\subsection{Hybrid batching}
Vertex and edge batching exhibit complementary properties for 
problems with varying difficulty.
Yet, when a query~$\calQ$ is given, the hardness of the problem is not known a-priori.
In this section we propose a hybrid approach that exhibits favourable properties, regardless of the hardness of the problem.

This hybrid batching strategy commences by searching over a graph $\calG(n_0,r_0)$
where $n_0$ is the same as for vertex batching and  $r_0 = O(f(n_0))$. 
As long as $n_i < n$, the next batch has 
$n_{i+1} = \eta_{v}n_{i}$ and $r_{i} = O(f(n_{i}))$. 
When $n_{i} = n$ (and $r_{i} = O(f(n))$),
all subsequent batches are similar to edge batching, i.e.,
$r_{i+1} = \eta_{e} r_{i}$ (and $n_{i+1} = n$).

This can be visualized on the space of subgraphs as sampling along the curve $f(|V|)$ from $|V| = n_0$ until $f(|V|)$ intersects $|V| = n$ and then sampling along the vertical line~$|V| = n$.
See~\figref{fig:ve_batching} and~\figref{fig:ve_comparitive} for a mental picture.
As we will see in our experiments, hybrid batching typically performs comparably (in terms of path quality) to vertex batching on easy problems and to edge batching on hard problems.

\section{Analysis for Halton Sequences}
\label{sec:instantiation}

In this section we consider the space of subgraphs and the densification strategies that we introduced in \sref{sec:approach} for the specific case that $\calS$ is a Halton sequence.
We start by describing the boundaries of the starvation regions.
We then continue by simulating the bound on the quality of the solution obtained as a function of the work done for each of our strategies.

\subsection{Starvation-region bounds}
To bound the vertex starvation region we wish to find~$n_{\text{min}}$ after which bounded sub-optimality can be guaranteed
to find the first solution. 
Note that $\delta^*$ is the clearance of the shortest path $\gamma*$ in $\calG$ connecting $s_1$ and $s_2$,
that~$p_d$ denotes the $d^{th}$ prime and $D_n \leq p_d / n^{1/d}$ for Halton sequences.
For~\eref{eq:dispersion_suboptimality} to hold we require that $2D_{n_{\text{min}}} < \delta^{*}$.
Thus,
$$
2D_{n_{\text{min}}} < \delta^{*} 
\Rightarrow
2\frac{p_d}{n_{\text{min}}^{1/d}} < \delta^{*}
\Rightarrow
n_{\text{min}} > \left( \frac{2p_d}{\delta^*}\right)^d. 
$$
Indeed, one can see that as the problem becomes harder (namely, $\delta^*$ decreases),
$n_{\text{min}}$ and 
the entire vertex-starvation region grows. 

We now  show that for Halton sequences, the edge-starvation region has a linear boundary, i.e. $f(|V|) = O(|V|)$.
Using \eref{eq:dispersion_suboptimality} we have that the minimal radius $r_{\min}(|V|)$ required for a graph with $|V|$ vertices is 
$$
r_{\min}(|V|) > 2D_{|V|} 
\Rightarrow 
r_{\min}(|V|) > \frac{2 p_d}{(|V|)^{1/d}}.
$$
For any $r$-disk graph~$\calG(\ell, r)$, the number of edges is $|E_{\ell, r}| = O\left(\ell^2 \cdot r^{d}\right)$.
In our case,
$$
f(|V|) 
= O\left(|V|^2 \cdot r_{\min}^{d}(|V|)\right)
= O(|V|).
$$

\begin{figure}
    \centering
    \begin{subfigure}[b]{0.49\columnwidth}
    \centering
        \includegraphics[width=\textwidth]{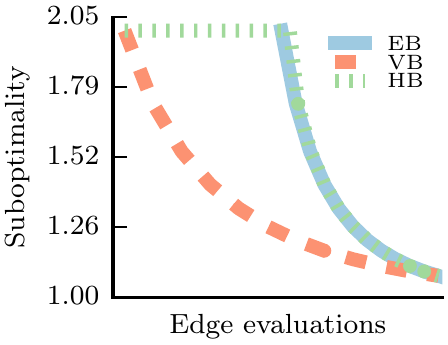}
        \caption{ Easy problem}
        \label{fig:ratio_easy}
    \end{subfigure}
    \begin{subfigure}[b]{0.49\columnwidth}
    \centering
        \includegraphics[width=\textwidth]{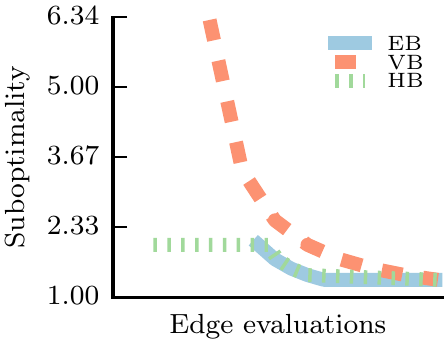}
        \caption{ Hard problem}
        \label{fig:ratio_hard}
    \end{subfigure}

    \caption{ A simulation of the work-suboptimality tradeoff for vertex, edge  and hybrid batching. 
    Here we chose $n = 10^6$ and $d=4$. 
    The easy and hard problems have
    $\delta^{*} = \sqrt{d} /2$ and $\delta^{*} = 5D_{n}$, respectively. 
    The plot is produced by sampling points along the curves $|V| = n$ and $|E| = |V| \cdot (|V|-1|)/2$ and using the respective values in \eref{eq:dispersion_suboptimality}.
    Note that $x$-axis is in log-scale.
}
    \label{fig:ratio}
\end{figure}

\subsection{Effort-to-quality ratio}
We now compare our densification strategies in terms of their worst-case anytime performance. Specifically, we plot the cumulative amount of work as subgraphs are searched, measured by the maximum number of edges that may be evaluated, as a function of the bound on the quality of the solution that may be obtained using \eref{eq:dispersion_suboptimality}.
We fix a specific setting (namely $d$ and $n$) and simulate the work done and the suboptimality using the necessary formulae.
This is done for an easy and a hard problem. See~\figref{fig:ratio}. 

Indeed, this simulation coincides with our discussion on properties of both batching strategies with respect to the problem difficulty.
Vertex batching outperforms edge batching on easy problems and vice versa. Hybrid batching lies somewhere in between the two approaches
with the specifics depending on problem difficulty.

\section{IMPLEMENTATION}
\label{sec:implementation}

\subsection{Search Parameters}
We choose the parameters for each densification strategy such that the number of batches is $O(\text{log}_2n)$.

\subsubsection{Edge Batching}
We set $\eta_e = 2^{1/d}$ . 
Recall that for $r$-disk graphs, the average degree of vertices is~$n \cdot r_{i}^{d}$, therefore this value (and hence the number of edges) is doubled after each iteration. 
We set $r_0 = 3\cdot n^{-1/d}$.

\subsubsection{Vertex Batching}
We set the initial number of vertices~$n_0$ to be $100$,
irrespective of the roadmap size and problem setting,
and set $\eta_v = 2$.
After each batch
we double the number of vertices.

\subsubsection{Hybrid Batching}
The parameters are derived from those used for vertex and edge batching.
We begin with $n_0 = 100$, and after each batch we increase the vertices
by a factor of~$\eta_v = 2$. For these searches, i.e. in the region where $n_i < n$,
we use $r_i = 3 \cdot n^{-1/d}$. This ensures the same radius
at $n$ as for edge batching. Subsequently, we increase the radius as
$r_i = \eta_e r_{i-1}$, where $\eta_e = 2^{1/d}$.

\subsection{Optimizations}

Our analysis and intuition is agnostic to any specific algorithms or implementations. 
However, for these densification strategies to be useful 
in practice, we employ certain optimizations.

\subsubsection{Search Technique}

Each subgraph is searched using Lazy $\text{A}^{*}$~\cite{CPL14} with 
incremental rewiring as in $\text{LPA}^{*}$~\cite{koenig2004lifelong}. For details,
see the search algorithm used for a single batch of $\text{BIT}^{*}$~\cite{GSB15}.
This lazy variant of $\text{A}^{*}$ has been shown to outperform other path-planning techniques for motion-planning search problems
with expensive edge evaluations~\cite{dellin2016unifying}.

\subsubsection{Caching Collision Checks}
Each time the collision-detector~$\calM$ is called for an edge, we store the ID of the edge along with the result using a hashing data structure. 
Subsequent calls for that specific edge are simply lookups in the hashing data structure which incur negligible running time.
Thus,~$\calM$ is called for each edge at most once.

\begin{figure*}[h]
    \begin{subfigure}[b]{0.5\columnwidth}
        \centering
        \includegraphics[width=\textwidth]{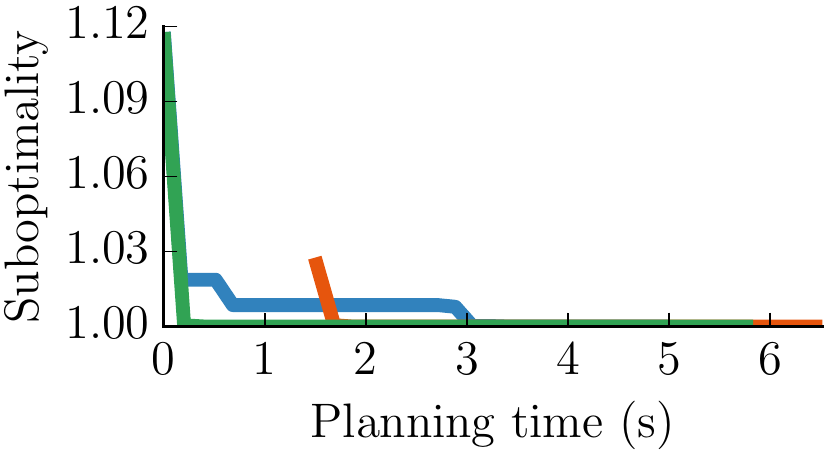}
        \caption{ $\mathbb{R}^{2}$ - Easy}
        \label{fig:results_2d_easy}
    \end{subfigure}
    \begin{subfigure}[b]{0.5\columnwidth}
        \centering
        \includegraphics[width=\textwidth]{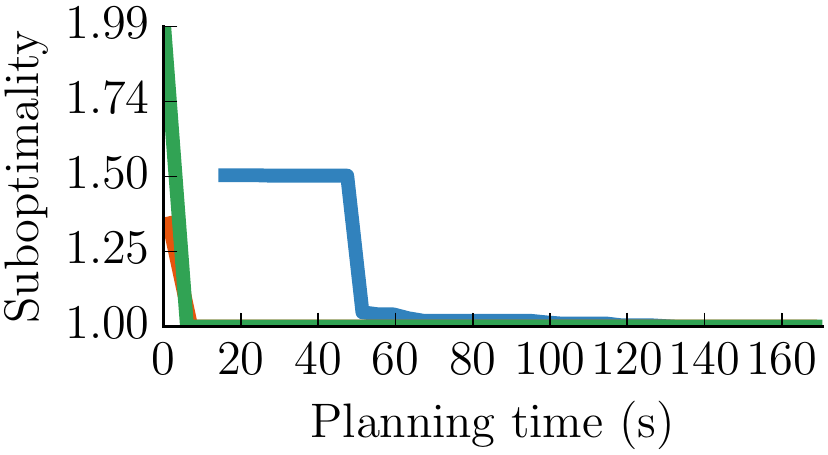}
        \caption{ $\mathbb{R}^{2}$ - Hard}
        \label{fig:results_2d_hard}
    \end{subfigure}
    \begin{subfigure}[b]{0.5\columnwidth}
        \centering
        \includegraphics[width=\textwidth]{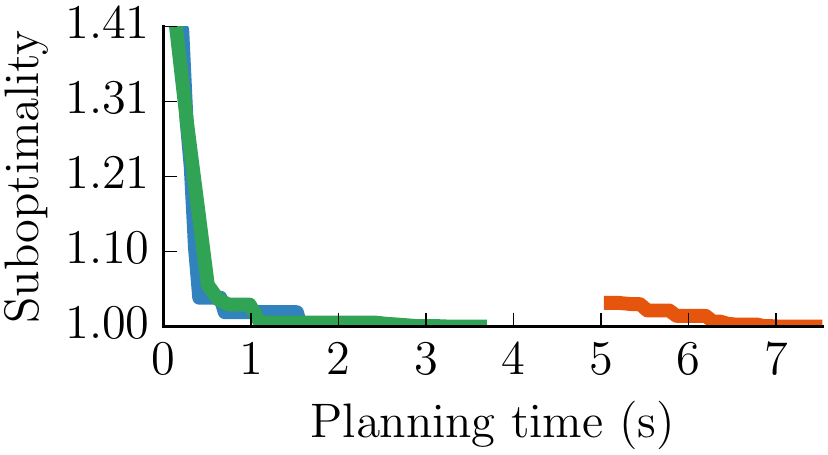}
        \caption{ $\mathbb{R}^{4}$ - Easy}
        \label{fig:results_4d_easy}
    \end{subfigure}
    \begin{subfigure}[b]{0.5\columnwidth}
        \centering
        \includegraphics[width=\textwidth]{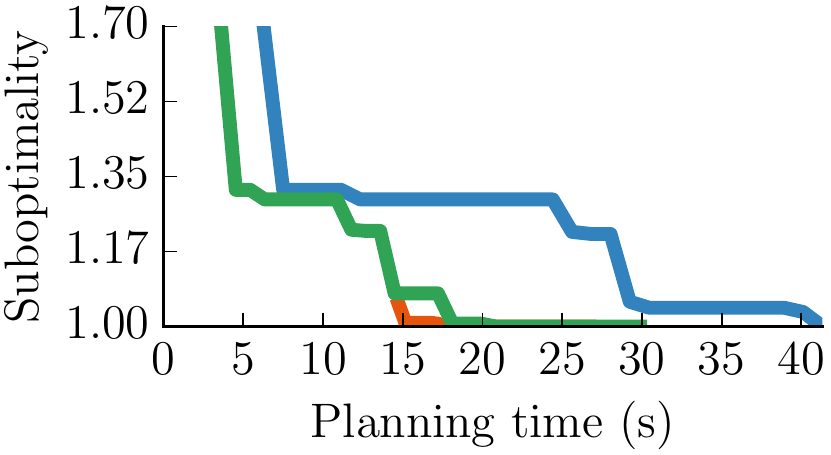}
        \caption{ $\mathbb{R}^{4}$ - Hard}
        \label{fig:results_4d_hard}
    \end{subfigure}

    \caption{ Experimental results in random unit hybercube scenarios for
    \protect\tikz{\protect\node[fill=myblue,draw=black]{};}\; vertex batching
    \protect\tikz{\protect\node[fill=myred,draw=black]{};}\; edge batching, and \protect\tikz{\protect\node[fill=mygreen,draw=black]{};}\; hybrid batching.
    The $y$-axis is the ratio between the length of the path produced by the algorithm and length of $\gamma^*$ (the shortest path
    on $\calG$) for that problem. The naive strategy of searching the complete graph requires the following times to find a solution
    - (\subref{fig:results_2d_easy}) $\mathbf{44}$s, (\subref{fig:results_2d_hard}) $\mathbf{200}$s, (\subref{fig:results_4d_easy}) $\mathbf{12}$s and (\subref{fig:results_4d_hard}) $\mathbf{56}$s. In each case this is significantly more than the time for any other strategy to reach the optimum.
    Figure best viewed in color.}

    \label{fig:results_2d_4d}
\end{figure*}

\subsubsection{Sample Pruning and Rejection}

For anytime algorithms, once an initial solution is obtained, subsequent searches should be focused
on the subset of states that could potentially improve the solution. When the space $\mathcal{X}$ is
Euclidean, this, so-called ``informed subset'', can be described 
by a prolate hyperspheroid \cite{GSB14}. 
For our densification strategies, we prune away all existing vertices
(for all batching), and reject the newer vertices (for vertex and hybrid batching),
that fall outside the informed subset. 

Successive prunings due to intermediate solutions significantly reduces the average-case complexity
of future searches \cite{GSB15}, despite the extra time required to do so, which is accounted for in our benchmarking. Note that for Vertex and Hybrid Batching, which begin with only a few samples, samples in successive batches that are outside the current ellipse can just be rejected. This is cheaper than pruning, which is required  for Edge Batching. Across all test cases, we noticed poorer performance when pruning was omitted.

In the presence of obstacles, the extent to which the complexity is reduced due to pruning is difficult to obtain analytically. As shown in Theorem \ref{th:ellipse}, however, in the assumption of free space, we can derive results for Edge Batching. This motivates using this heuristic.

\begin{theorem}
\label{th:ellipse}
Running edge batching in an obstacle-free $d$-dimensional Euclidean space over
a roadmap constructed using a deterministic low-dispersion sequence with
$r_0 > 2D_{n}$ and $r_{i+1} = 2^{1/d} r_{i}$,
while 
using sample pruning and rejection 
makes the worst-case complexity
of the total search, measured in edge evaluations, $O(n^{1+ 1/d})$.
\end{theorem}

\begin{proof}
Let $\cbest{i}$ denote the cost of the solution obtained after $i$ iterations by our edge batching algorithm,
and $\cmin = \rho(s_1, s_2) \leq \sqrt{d}$ denote the cost of the optimal solution.
Using \eref{eq:dispersion_suboptimality},
\begin{equation}
\label{eq:cbest}
\cbest{i} \leq \left( 1 + \varepsilon_i \right) \cmin,
\end{equation}
where 
$\varepsilon_i = \frac{2D_n}{r_i - 2D_n}$.
Using the parameters for edge batching,
\begin{equation}
\label{eq:epsdec}
\eps{i+1} = \frac{2D_n}{r_{i+1} - 2D_n} = \frac{2D_n}{2^\frac{1}{d} r_i - 2D_n} \leq \frac{\eps{i}}{2^\frac{1}{d}}.
\end{equation}

Let $\imax$ be the maximum number of iterations and recall that we have $\imax = O \left( \log_2 n \right)$.

Note that the fact that vertices and edges are pruned away, does not change the bound provided in \eref{eq:cbest}.
To compute the actual number of edges considered at the $i$th iteration, we bound the volume of the prolate hyperspheriod $\mathcal{X}_{\cbest{i}}$ in $\mathbb{R}^{d}$ (see~\cite{GSB14}) by,
\begin{equation}
    \label{eq:ellipse}
    \vol{\mathcal{X}_{\cbest{i}}} 
    = \frac{ \cbest{i} \left( \left(\cbest{i}\right)^2 - \cmin^2 \right)^{\frac{d-1}{2}} \xi_d}{2^d},
\end{equation}
where $\xi_{d}$ is the volume of an $\mathbb{R}^{d}$ unit-ball. 
Using \eref{eq:cbest} in \eref{eq:ellipse},
\begin{equation}
    \vol{\mathcal{X}_{\cbest{i}}} 
    \leq \eps{i}^{\frac{d-1}{2}} \left(1 + \eps{i} \right)\left(2 + \eps{i}\right)^{\frac{d-1}{2}}  \const_d, 
\end{equation}
where $\const_d = \xi_d \cdot \left( \cmin / 2\right)^d$
 is a constant.
Using \eref{eq:epsdec} we can bound the volume of the ellipse used at the $i$'th iteration, where $i \geq 1$,
\begin{equation}
\begin{aligned}
    \vol{\mathcal{X}_{\cbest{i}}}   & 
    \leq 
    \eps{i}^{\frac{d-1}{2}} \left(1 + \eps{0} \right)\left(2 + \eps{0}\right)^{\frac{d-1}{2}} \const_d \\
                                    & 
    \leq 
    \eta^{-\frac{i(d-1)}{2}} \eps{0}^{\frac{d-1}{2}} \left(1 + \eps{0} \right)\left(2 + \eps{0}\right)^{\frac{d-1}{2}} \const_d \\
                                    & 
    \leq 
    2^{-\frac{i(d-1)}{2d}} \vol{\mathcal{X}_{\cbest{0}}}
\end{aligned}
\end{equation}
Furthermore, we choose  $r_0$ such that $\vol{\mathcal{X}_{\cbest{0}}} \leq \vol{\mathcal{X}}$.
Now, the number of vertices in $\spacebest$ can be bounded by,
\begin{equation}
n_{i+1} = \frac{ \vol{\spacebest} }{ \vol{\mathcal{X}} } n \leq 2^{-\frac{i(d-1)}{2d}} n.
\end{equation}
Recall that we measure the amount of work done by the search at iteration $i$ using $|E_i|$, the number of edges considered. Thus, 
\begin{equation}
|E_i| = O\left(n_i^2 r_i^d \right) = O\left(n^2 2^{-\frac{i(d-1)}{d}} \left(r_0 2^{\frac{i}{d}} \right)^d \right) = O\left(n 2^{\frac{i}{d}} \right)
\end{equation}
Finally, the total work done by the search over all iterations is
\begin{equation}
O \left( \sum \limits_{i=0}^{\log_2 n} n 2^\frac{i}{d} \right)
=
O \left( n \sum \limits_{i=0}^{\log_2 n} 2^{i/d} \right)
= 
O \left( n^{1+\frac{1}{d}} \right).
\end{equation}

\end{proof}

\begin{figure*}[tbh]
    \begin{subfigure}[b]{0.5\columnwidth}
        \centering
        \includegraphics[width=\textwidth]{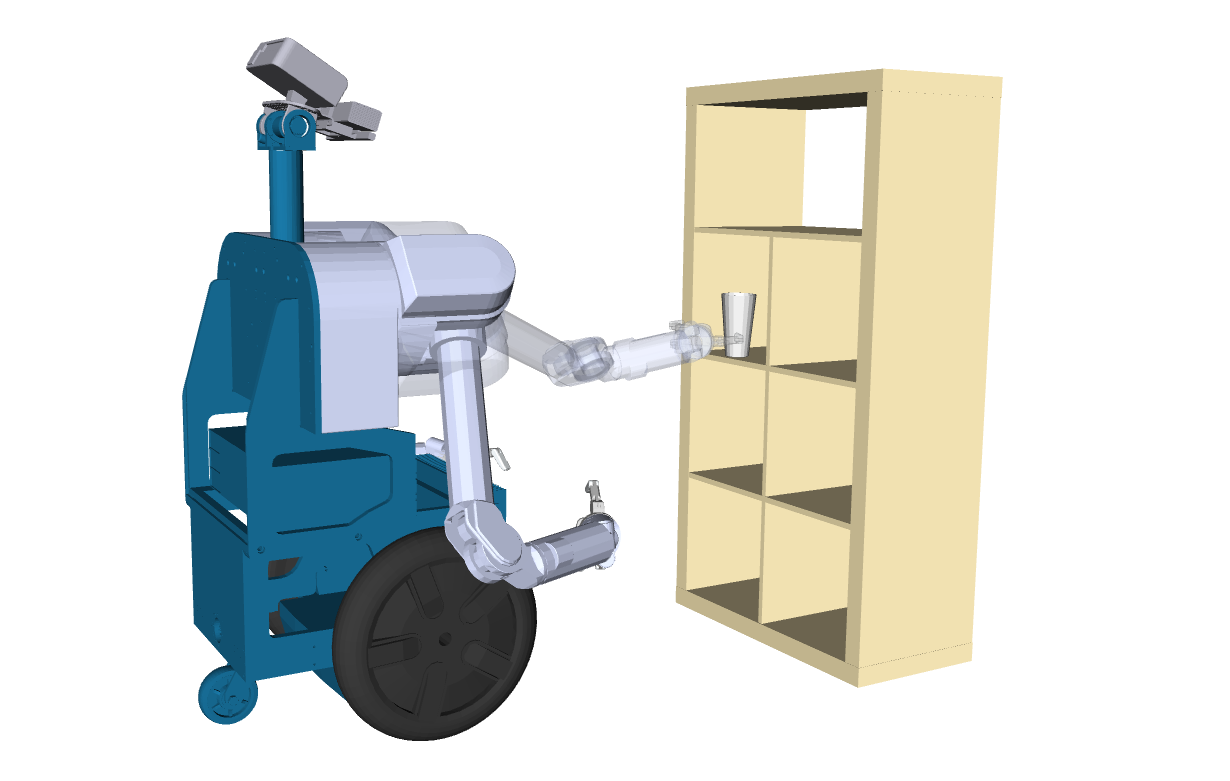}
        \caption{}
        \label{fig:herbprob1}
    \end{subfigure}
    \begin{subfigure}[b]{0.5\columnwidth}
        \centering
        \includegraphics[width=\textwidth]{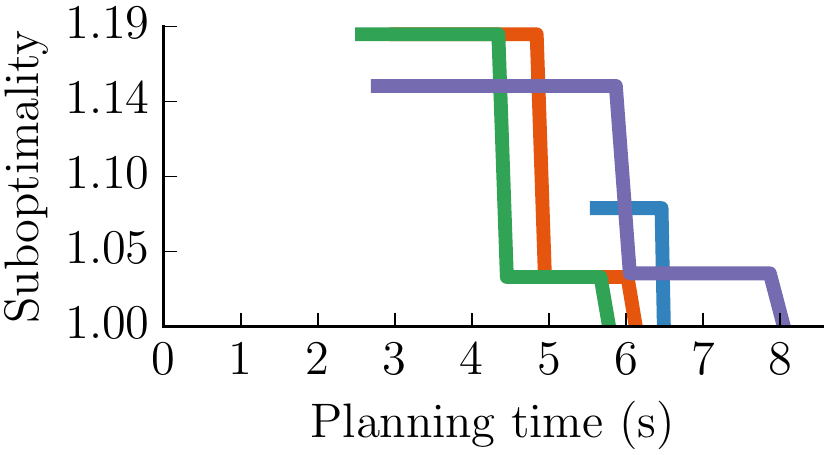}
        \caption{}
        \label{fig:herbplot1}
    \end{subfigure}
    \begin{subfigure}[b]{0.5\columnwidth}
        \centering
        \includegraphics[width=\textwidth]{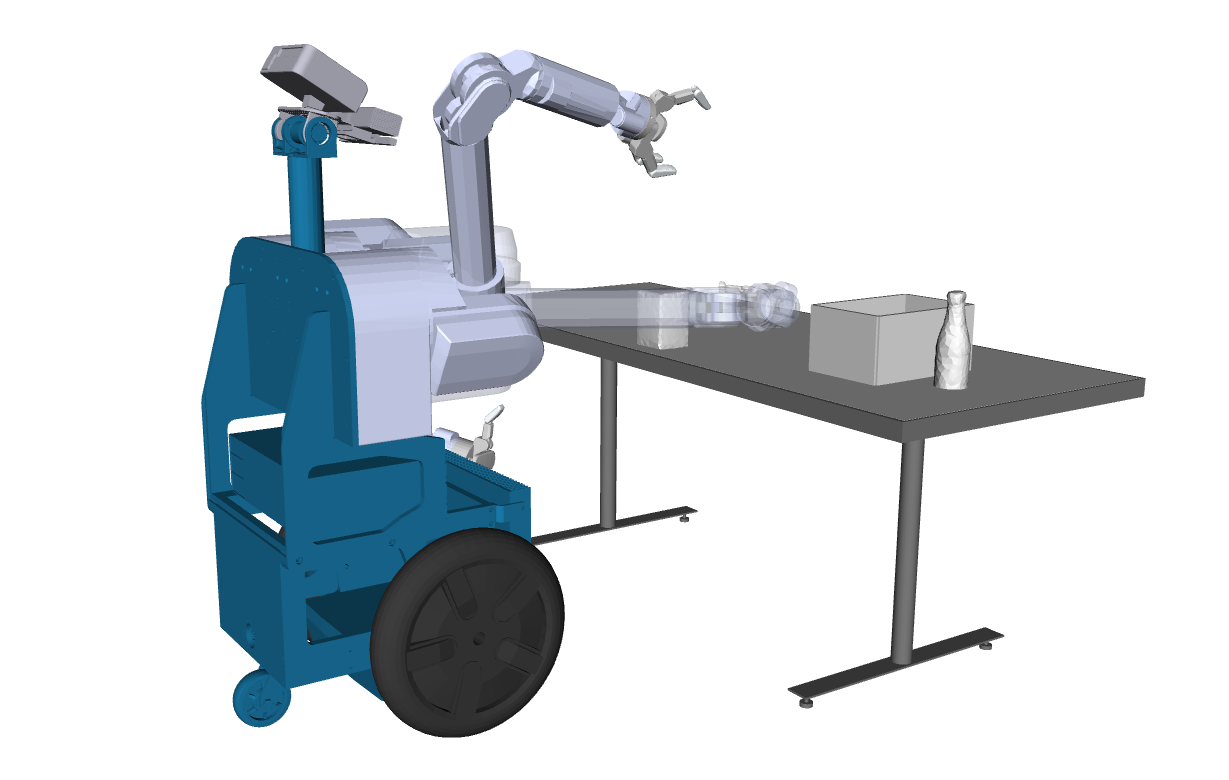}
        \caption{}
        \label{fig:herbprob2}
    \end{subfigure}
    \begin{subfigure}[b]{0.5\columnwidth}
        \centering
        \includegraphics[width=\textwidth]{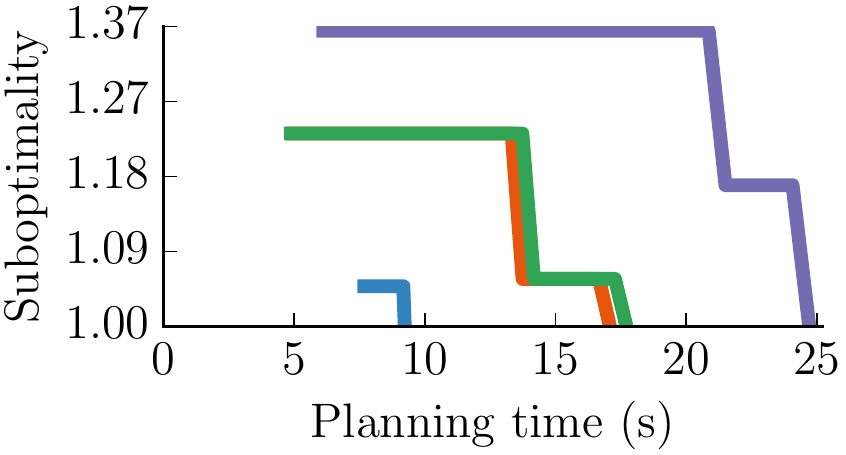}
        \caption{}
        \label{fig:herbplot2}
    \end{subfigure}

    \caption{ We show results on 2 manipulation problems for \protect\tikz{\protect\node[fill=myblue,draw=black]{};}\; vertex batching
    \protect\tikz{\protect\node[fill=myred,draw=black]{};}\; edge batching, \protect\tikz{\protect\node[fill=mygreen,draw=black]{};}\; hybrid batching
    and \protect\tikz{\protect\node[fill=mypurple,draw=black]{};}\; $\text{BIT}^{*}$. For each problem the goal configuration of the right
    arm is rendered translucent. Both of the problems are fairly constrained and non-trivial. The problem depicted in (\subref{fig:herbprob2}) has a 
    large clear area in front of the starting configuration, which may allow for a long edge. This could explain the better performance
    of vertex batching. The naive strategy takes $25$s for (\subref{fig:herbplot1}) and $44$s for (\subref{fig:herbplot2}) respectively.}

    \label{fig:results_herb}
\end{figure*}

\section{EXPERIMENTS}
\label{sec:experiments}
Our implementations of the various strategies are based on the publicly available
OMPL \cite{SMK12} implementation of $\text{BIT}^{*}$~\cite{GSB15}. Other than the specific parameters
and optimizations mentioned earlier, we use the default parameters of $\text{BIT}^{*}$. 
Notably, we use the Euclidean distance heuristic, an approximately sorted
queue, and limit graph pruning to changes in path length greater than 1\%.


\subsection{Random scenarios}

The different batching strategies are compared to each other on problems in $\mathbb{R}^{d}$ for $d =2,4$.
The domain is the unit hypercube $[0,1]^{d}$ while the obstacles are randomly generated axis-aligned
$d$-dimensional hyper-rectangles. 
All problems have a start configuration of $[0.25,0.25,\ldots]$
and a goal configuration of $[0.75,0.75,\ldots]$.
We used the first $n = 10^{4}$ and $n = 10^{5}$ points of the Halton sequence for the~$\mathbb{R}^{2}$ 
and~$\mathbb{R}^{4}$ problems, respectively.

Two parameters of the obstacles are varied to approximate the notion of problem hardness
described earlier -- the number of obstacles and the fraction
of $\mathcal{X}$ which is in $\mathcal{X}_{obs}$, which we denote by $\zeta_{\text{obs}}$. 
Specifically, in $\mathbb{R}^{2}$, we have 
easy problems with $100$ obstacles and $\zeta_{\text{obs}} = 0.33$, and 
hard problems with $1000$ obstacles and $\zeta_{\text{obs}} = 0.75$. 
In $\mathbb{R}^{4}$
we maintain the same values for $\zeta_{\text{obs}}$, but use $500$ and $3000$ obstacles for easy and hard problems, respectively. For each problem setting ($\mathbb{R}^{2}$/$\mathbb{R}^{4}$; easy/hard) we generate~$30$ different random
scenarios and evaluate each strategy with the same set of samples on each of them. 
Each random scenario has a different set of solutions, so we show a representative
result for each problem setting in \figref{fig:results_2d_4d}. 

The results align well with our intuition
about the relative performance of the densification strategies on easy and hard problems.
Notice that the naive strategy of searching $\calG$ with $\text{A}^{*}$ directly requires considerably more time to report the 
optimum solution than any other strategy. We mention the numbers 
in the accompanying caption of \figref{fig:results_2d_4d} but avoid plotting them so
as not to stretch the figures. Note the reasonable performance of hybrid batching across problems and difficulty levels. 

\subsection{Manipulation problems}

We also run simulated experiments on HERB \cite{srinivasa2010herb}, a mobile manipulator designed and built by the Personal Robotics Lab at Carnegie Mellon University. 
The planning problems are for the 7-DOF right arm, on the problem
scenarios shown in Fig.~\ref{fig:results_herb}. 
We use a roadmap of $10^5$ vertices defined using a Halton sequence~$\calS$ which was generated using the first $7$ prime numbers.
In addition to the batching strategies, we also evaluate the performance
of $\text{BIT}^{*}$~\cite{GSB15}, using the same set of samples~$\calS$. $\text{BIT}^{*}$ had been shown to achieve anytime performance
superior to contemporary anytime algorithms. 
The hardness of the problems in terms of clearance
is difficult to visualize in terms of the C-space of the arm, but the goal regions
are considerably constrained. As our results show (Fig.~\ref{fig:results_herb}), all densification
strategies solve the difficult planning problem in reasonable time, and generally outperform the BIT* strategy
on the same set of samples.

\section{CONCLUSION AND FUTURE WORK}
\label{sec:conclusion}

We present, analyze and implement several densification strategies for anytime planning on large E$^4$ graphs. 
We provide theoretical motivation for these densification techniques, and show that they outperform the naive approach 
significantly on difficult planning problems.

In this work we demonstrate our analysis for the case where the set of samples is generated from a low-dispersion
deterministic sequence. 
A natural extension is to provide a similar analysis for a sequence of random i.i.d. samples.
Here, $f(|V|) = O\left(\log |V| \right)$~\cite{KF11}
instead of $O(|V|)$.
When out of the starvation regions we would like to bound the quality obtained similar to the bounds provided by~\eref{eq:dispersion_suboptimality}.
A starting point would be to leverage recent results by Dobson et. al.~\cite{DMB15} for Random Geometric Graphs
under expectation, albeit for a \emph{specific} radius $r$.

Another question we wish to pursue is alternative possibilities to traverse the subgraph space of $\calG$.
As depicted in Fig.~\ref{fig:ve_batching}, our densification strategies are essentially
ways to traverse this space . 
We discuss three techniques that traverse relevant boundaries of the space. But there are innumerable
trajectories that a strategy can follow to reach the optimum. It would be interesting
to compare our current batching methods, both theoretically and practically, to those that
go through the interior of the space.







\end{document}